\documentclass[journal]{IEEEtran}

\newtheorem{thm}{Theorem}[section]

\newtheorem{lem}[thm]{Lemma}
\newtheorem{prop}[thm]{Proposition}
\newtheorem{problem}{Problem}
\newtheorem{ex}[thm]{Example}
\newtheorem{defn}[thm]{Definition}
\newtheorem{rem}[thm]{Remark}

\newcommand{\set}[1]{\left\{#1\right\}}

\IEEEoverridecommandlockouts
\usepackage{graphicx} 

\usepackage{epsfig} 
\usepackage{amsmath}
\usepackage{amssymb}
\usepackage{epstopdf}
\usepackage{cite}
\usepackage[noend,ruled,linesnumbered]{algorithm2e}
\SetKwComment{Comment}{$\triangleright$\ } {}
\usepackage{multirow}
\usepackage{rotating}
\usepackage{subfigure}
\usepackage{color}
\usepackage{mysymbol}
\usepackage{url}

\begin{document}
\bstctlcite{IEEEexample:BSTcontrol}
\title{Temporal Logic Task Planning and Intermittent Connectivity Control of Mobile Robot Networks}

\author{Yiannis~Kantaros,~\IEEEmembership{Student Member,~IEEE,} Meng~Guo,~\IEEEmembership{Student Member,~IEEE,} and Michael~M.~Zavlanos,~\IEEEmembership{Member,~IEEE}
\thanks{Yiannis Kantaros, Meng Guo, and Michael M. Zavlanos are with the Department of Mechanical Engineering and Materials Science, Duke University, Durham, NC 27708, USA. $\left\{\text{yiannis.kantaros, meng.guo, michael.zavlanos}\right\}$@duke.edu. This work is supported in part by NSF under grant CNS \#1302284 and by ONR under grant \#N000141812374.}
}

\maketitle
\thispagestyle{empty}
\pagestyle{empty}
\begin{abstract}
In this paper, we develop a distributed intermittent communication and task planning framework for mobile robot teams. The goal of the robots is to accomplish complex tasks, captured by local Linear Temporal Logic formulas, and share the collected information with all other robots and possibly also with a user. Specifically, we consider situations where the robot communication capabilities are not sufficient to form reliable and connected networks while the robots move to accomplish their tasks. In this case, intermittent communication protocols are necessary that allow the robots to temporarily disconnect from the network in order to accomplish their tasks free of communication constraints. We assume that the robots can only communicate with each other when they meet at common locations in space. Our distributed control framework jointly determines local plans that allow all robots fulfill their assigned temporal tasks, sequences of communication events that guarantee information exchange infinitely often, and optimal communication locations that minimize red a desired distance metric. Simulation results verify the efficacy of the proposed controllers.
\end{abstract}


\section{Introduction}

\IEEEPARstart{R}{ecently}, there has been a large amount of work focused on designing controllers that ensure point-to-point or end-to-end network connectivity of mobile robot networks for all time. Such controllers either rely on graph theoretic approaches  \cite{zavlanos2007potential,ji2007distributed,Zavlanos_IEEETRO08,sabattini2013decentralized,Zavlanos_IEEE11} or employ more realistic communication models that take into account path loss, shadowing, and multi-path fading as well as optimal routing decisions for desired information rates \cite{zavlanos2013network,yan2012robotic,kantaros2016distributed,kantaros2016global,stephan2017concurrent}. 
%
%
However, due to the uncertainty in the wireless channel, it is often impossible to ensure all-time connectivity in practice. Moreover, such methods often prevent the robots from accomplishing their tasks, as motion planning is always restricted by connectivity constraints on the network. Therefore, a much preferred solution is to allow robots to communicate in an intermittent fashion and operate in disconnect mode the rest of the time. 

Intermittent communication in multi-agent systems has been studied in consensus problems \cite{wen2014distributed}, coverage problems \cite{wang2010awareness}, and in delay-tolerant networks \cite{lindgren2003probabilistic, jones2007practical}. The common assumption in these works is that the communication network is connected over time, infinitely often. Relevant is also the work on event-based network control \cite{dimarogonas2012distributed,tabuada2007event} where, although the network is assumed to be connected for all time, messages between the agents are exchanged intermittently when certain events take place. In this paper, we lift all connectivity assumptions and, instead, control the communication network itself so that it is guaranteed to be intermittently connected, infinitely often.
%
%
%
Specifically, we assume that robots can only communicate when they are physically close to each other. The intermittent connectivity requirement is captured by
a global Linear Temporal Logic (LTL) statement that forces small groups of robots, also called teams, to meet infinitely often at locations in space that are common for each team, but possibly different across teams. We assume that every robot belongs to at least one team and that there is a path, i.e., a sequence of teams where consecutive teams have non-empty intersections, connecting every two teams of robots, so that
information can propagate in the network.

In addition to the intermittent communication requirement, we also assume that the robots are responsible for accomplishing independent tasks that are specified by local LTL formulas. These tasks can be, e.g., gathering of information in the environment that needs to reach all other robots and possibly a user through the proposed intermittently connected network. Given the global LTL statement comprised of the intermittent communication requirement and the local LTL tasks, existing control synthesis approaches for global LTL specifications \cite{kloetzer2008distributed,chen2011synthesis,chen2012formal} that rely on transition systems to abstract robot mobility can be used to obtain correct-by-construction controllers. Nevertheless, such approaches do not optimize task performance. Optimal control synthesis algorithms for mobile robot networks under global LTL specifications are proposed in \cite{ulusoy2013optimality,ulusoy2014optimal,kantaros2017Csampling,kantaros2017Dsampling}. Common in \cite{ulusoy2013optimality,ulusoy2014optimal} is that they rely on the construction of a synchronous product automaton among all robots and the application of graph search methods to synthesize optimal plans. Therefore, these approaches are resource demanding and scale poorly with the number of robots. 
%
Sampling-based optimal control synthesis methods under global LTL specifications have also been proposed by the authors in \cite{kantaros2017Csampling} that scale better than the methods in \cite{ulusoy2013optimality,ulusoy2014optimal}. 
The methods proposed in \cite{ulusoy2013optimality,ulusoy2014optimal,kantaros2017Csampling} are all centralized and offline and, therefore, not reactive to new tasks. Moreover, they require as an input the B$\ddot{\text{u}}$chi automaton that corresponds to the global LTL formula, which is generated by a computationally expensive process.
A distributed implementation of \cite{kantaros2017Csampling} that can optimize feasible motion plans online is presented in \cite{kantaros2017Dsampling}. However, \cite{kantaros2017Dsampling} requires an all-time connected communication network which is not the case here. A new logic, called counting linear temporal logic, is proposed in \cite{sahin2017provably} that can be used for coordination of large collections of agents. However, this approach is centralized, offline, and assumes that the identity of the agents is not important for the successful accomplishment of the task, which is not the case here due to the intermittent connectivity requirement. 

In this work, our goal is to synthesize motion plans for all robots so that both the local LTL tasks and the global LTL formula capturing the intermittent connectivity requirement are satisfied, while minimizing a desired distance metric. To achieve that, we avoid the construction of the product automaton altogether and instead propose an online and distributed framework to design correct-by-construction controllers for the robots. 
In particular, we first focus on the intermittent connectivity requirement and propose a new distributed framework to design sequences of communication events, also called communication schedules,  for all teams of robots. Then, we develop discrete plans for the robots that satisfy the local LTL tasks while ensuring that teams can communicate according to the predetermined schedules. The locations of the communication events in these discrete plans are selected so that they optimize a desired distance metric. The proposed controllers are synthesized in a distributed and online fashion, and can be executed asynchronously, which is not the case in most relevant literature as, e.g., in \cite{kloetzer2010automatic,kantaros2017Csampling,kantaros2017Dsampling,sahin2017provably,pola2016decentralized}. 

To the best of our knowledge, the most relevant works to the one proposed here are recent works by the authors \cite{kantaros2016distributedInterm,kantaros2016simultaneous,kantaros15asilomar,guo2017distributed}. Specifically, \cite{kantaros2016distributedInterm} proposes an asynchronous distributed intermittent communication framework that is a special case of the one proposed here in that every robot belongs to exactly two teams and the robots in every team can only meet at a single predetermined location. This framework is extended in \cite{kantaros2016simultaneous}, where robots can belong to any number of teams and every team can select among multiple locations to meet, same as in the work considered here. Nevertheless, neither of the approaches in \cite{kantaros2016distributedInterm,kantaros2016simultaneous} consider concurrent task planning. Intermittent communication control and task planning is considered in \cite{kantaros15asilomar} that relies on the construction of a synchronous product automaton among all robots and, therefore, this approach is centralized and does not scale well with the number of robots. A distributed online approach to this problem is proposed in \cite{guo2017distributed}. The method proposed here is more general in that it can handle the data gathering tasks and the star communication topology in \cite{guo2017distributed} that considers information flow only to the root/user. In fact, in the proposed method, information can flow intermittently between any pair of robots and possibly a user in a multi-hop fashion. Another fundamental difference with \cite{guo2017distributed} is that here the robots first decide \textit{how} they want to communicate by constructing abstract schedules of communication events and then decide \textit{where} they want to communicate by embedding online and optimally these schedules in the workspace so that the desired tasks are also satisfied. In fact, this is a unique feature of the proposed approach that differentiates it from existing literature on communication control where communication is always state-dependent. Other relevant methods that do not rely on LTL for intermittent communication control are presented in \cite{zavlanos2010synchronous,hollinger2010multi}. However, these methods impose strong restrictions on the communication pattern that can be achieved, while \cite{zavlanos2010synchronous} also does not consider concurrent task planning. We provide theoretical guarantees supporting the proposed framework, as well as numerical simulations showing its ability to solve very large and complex planning problems that existing model checking techniques cannot solve.  To the best of our knowledge, this is the first  distributed, online, and asynchronous framework for temporal logic path planning and intermittent communication control that can be applied to large-scale multi-robot systems.

The rest of this paper is organized as follows. In Section \ref{sec:prel} we present some preliminaries in LTL. The problem formulation is described in Section \ref{sec:prob}. In Section \ref{sec:commun}, we design a distributed schedules of communication events that ensure intermittent connectivity. In Section \ref{sec:integration}, we design discrete motion plans that satisfy the assigned local LTL tasks and the intermittent connectivity requirement as per the communication schedules, while minimizing a distance metric. Theoretical guarantees of the proposed algorithm are presented in Section \ref{sec:analysis}. Simulation results are included in Section \ref{sec:sim}. 

\section{Preliminaries}\label{sec:prel} 
The basic ingredients of Linear Temporal Logic are a set of atomic propositions $\mathcal{AP}$, the boolean operators, i.e., conjunction $\wedge$, and negation $\neg$, and two temporal operators, next $\bigcirc$ and until $\mathcal{U}$. LTL formulas over a set $\mathcal{AP}$ can be constructed based on the following grammar: $\phi::=\text{true}~|~\pi~|~\phi_1\wedge\phi_2~|~\neg\phi~|~\bigcirc\phi~|~\phi_1~\mathcal{U}~\phi_2$, where $\pi\in\mathcal{AP}$. For the sake of brevity we abstain from presenting the derivations of other Boolean and temporal operators, e.g., \textit{always} $\square$, \textit{eventually} $\lozenge$, \textit{implication} $\Rightarrow$, which can be found in \cite{baier2008principles}. An infinite \textit{word} $\sigma$ over the alphabet $2^{\mathcal{AP}}$ is defined as an infinite sequence  $\sigma=\pi_0\pi_1\pi_2\dots\in (2^{\mathcal{AP}})^{\omega}$, where $\omega$ denotes infinite repetition and $\pi_k\in2^{\mathcal{AP}}$, $\forall k\in\mathbb{N}$. The language $\texttt{Words}(\phi)=\left\{\sigma\in (2^{\mathcal{AP}})^{\omega}|\sigma\models\phi\right\}$ is defined as the set of words that satisfy the LTL formula $\phi$, where $\models\subseteq (2^{\mathcal{AP}})^{\omega}\times\phi$ is the satisfaction relation.

Any LTL formula $\phi$ can be translated into a Nondeterministic B$\ddot{\text{u}}$chi Automaton (NBA) over $2^{\mathcal{AP}}$ denoted by $B$, which is defined as follows \cite{vardi1986automata}:
\begin{defn}[NBA]
A \textit{Nondeterministic B$\ddot{\text{u}}$chi Automaton} (NBA) $B$ over $2^{\mathcal{AP}}$ is defined as a tuple $B=\left(\ccalQ_{B}, \ccalQ_{B}^0,\Sigma,\rightarrow_B,\mathcal{F}_B\right)$, where $\ccalQ_{B}$ is the set of states, $\ccalQ_{B}^0\subseteq\ccalQ_{B}$ is a set of initial states, $\Sigma=2^{\mathcal{AP}}$ is an alphabet, $\rightarrow_{B}\subseteq\ccalQ_{B}\times \Sigma\times\ccalQ_{B}$ is the transition relation, and $\ccalF_B\subseteq\ccalQ_{B}$ is a set of accepting/final states. 
\end{defn}
An \textit{infinite run} $\rho_B$ of $B$ over an infinite word $\sigma=\pi_0\pi_1\pi_2\dots$, $\pi_k\in\Sigma=2^{\mathcal{AP}}$ $\forall k\in\mathbb{N}$ is a sequence $\rho_B=q_B^0q_B^1q_B^2\dots$ such that $q_B^0\in\ccalQ_B^0$ and $(q_B^{k},\pi_k,q_B^{k+1})\in\rightarrow_{B}$, $\forall k\in\mathbb{N}$. An infinite run $\rho_B$ is called \textit{accepting} if $\texttt{Inf}(\rho_B)\cap\ccalF_B\neq\varnothing$, where $\texttt{Inf}(\rho_B)$ represents the set of states that appear in $\rho_B$ infinitely often. The words $\sigma$ that result in an accepting run of $B$ constitute the accepted language of $B$, denoted by $\ccalL_B$. Then it is proven \cite{baier2008principles} that the accepted language of a NBA $B$, associated with an LTL formula $\phi$, is equivalent to the words of $\phi$, i.e., $\ccalL_B=\texttt{Words}(\phi)$.

\section{Problem Formulation}\label{sec:prob}

 
Consider $N\geq 1$ mobile robots operating in a workspace $\ccalW\subset\mathbb{R}^d$, $d\in\set{2,3}$, containing $W>0$ locations of interest denoted by $\bbv_j$, $j\in\ccalI:=\set{1,\dots,W}$. Mobility of robot $i\in\mathcal{N}:=\{1,\dots,N\}$ in $\ccalW$ is captured by a weighted Transition System (wTS) that is defined as follows:

\begin{defn}[weighted Transition System]\label{defn:wTS}
A weighted \textit{Transition System} for robot $i$, denoted by $\text{wTS}_{i}$ is a tuple $\text{wTS}_{i}=\left(\ccalQ_{i}, q_{i}^0,\rightarrow_{i}, w_i, \mathcal{AP},L_{i}\right)$ where (a) $\ccalQ_{i}=\{q_{i}^{\bbv_j}, j\in\ccalI\}$ is the set of states, where a state $q_{i}^{\bbv_j}$ indicates that robot $i$ is at location $\bbv_j\in\ccalW$; (b) $q_{i}^0\in\ccalQ_{i}$ is the initial state of robot $i$; (c) $\rightarrow_{i}\subseteq\ccalQ_{i}\times\ccalQ_{i}$ is a given transition relation such that $(q_i^{\bbv_j},q_i^{\bbv_e})\in\rightarrow_i$ if there exists a controller that can drive robot $i$ from location $\bbv_j$ to $\bbv_e$ in finite time without going through any other location $\bbv_c$; 
(d) $w_{i}:\ccalQ_{i}\times\ccalQ_{i}\rightarrow \mathbb{R}_+$ is a weight function that captures the distance that robot $i$ needs to travel to move from $\bbv_j$ to $\bbv_e$;\footnote{Note that alternative weights can be assigned to the transitions of the wTSs that can capture e.g.,the time, or energy required for robot $i$ to move from $\bbv_j$ to $\bbv_e$.}
(e) $\mathcal{AP}=\{\{\pi_{i}^{\bbv_j}\}_{i=1}^N\}_{j\in\ccalI}$ is the set of atomic propositions associated with each state; and 
(f) $L_{i}:\ccalQ_{i}\rightarrow \mathcal{AP}$ is defined as $L_i(q_i^{\bbv_j})=\pi_i^{\bbv_j}$, for all $i\in\ccalN$ and $j\in\ccalI$.
\end{defn} 

Every robot $i\in \ccalN$ is responsible for accomplishing high-level tasks associated with some of the locations $\bbv_j$, $j\in\ccalI$. Hereafter, we assume that the tasks assigned to the robots are independent from each other. Specifically, we assume that the task assigned to robot $i$ is captured by a local $\text{LTL}_{-\bigcirc}$ formula  $\phi_i$ \cite{clarke1999model} specified over the set of atomic propositions $\mathcal{AP}=\{\{\pi_{i}^{\bbv_j}\}_{i=1}^N\}_{j\in\ccalI}$, where $\pi_{i}^{\bbv_j}=1$ if $\left\|\bbx_{i}-\bbv_j\right\|\leq\epsilon$, for a sufficiently small $\epsilon>0$, and 0 otherwise, for all $i\in\ccalN$ and $j\in\ccalI$.\footnote{The syntax of $\text{LTL}_{-\bigcirc}$ is the same as the syntax of LTL excluding the `next' operator. The choice of $\text{LTL}_{-\bigcirc}$ over LTL is motivated by the fact that we are interested in the continuous time execution of the synthesized plans, in which case the next operator is not meaningful. This choice is common in relevant works, see, e.g., \cite{kloetzer2008fully} and the references therein.}
%
Namely, the atomic proposition $\pi_{i}^{\bbv_j}$ is true if robot $i$ is sufficiently close to location $\bbv_j$. For example, an $\text{LTL}_{-\bigcirc}$ task for robot $i$ can be: $\phi_i=(\square\Diamond \pi_i^{\bbv_4}) \wedge ((\neg\pi_i^{\bbv_4})\mathcal{U}\pi_i^{\bbv_8}) \wedge (\Diamond \pi_i^{\bbv_5}) \wedge (\square \neg \pi_i^{\bbv_3})\wedge (\square\Diamond \pi_i^{\bbv_1})$, which requires robot $i$ to (i) visit location $\bbv_4$ infinitely often, (ii) never visit location $\bbv_4$ until location $\bbv_8$ is visited, (iii) eventually visit location $\bbv_5$, (iv) always avoid an obstacle located at $\bbv_3$, and (v) visit location at $\bbv_1$ infinitely often. Together with accomplishing local tasks, robots are also responsible for communicating with each other so that any information that is collected as part of these tasks is propagated in the network and, possibly, eventually reaches a user. 

To define a communication network among the robots, we first partition  the robot team into $M\geq 1$ robot subgroups, called also teams, and require that every robot belongs to at least one subgroup. The indices $i$ of the robots that belong to the $m$-th subgroup are collected in a set denoted by $\mathcal{T}_m$, for all $m\in\mathcal{M}:=\{1,2,\dots,M\}$. We define the set that collects the indices of teams that robot $i$ belongs to as $\ccalM_i=\{m|i\in\ccalT_m,~m\in\mathcal{M}\}$. Also, for robot $i$ we define the set that collects the indices of all other robots that belong to common teams with robot $i$ as $\ccalN_i=\{j|j\in\ccalT_m, \forall m\in\ccalM_i\}\setminus\{i\}$, $\forall i\in \mathcal{N}$. 
%
Given the robot teams $\ccalT_m$, for all $m\in\mathcal{M}$, we can define the graph over these teams as follows. 

\begin{defn}[Team Membership Graph $\mathcal{G}_{\mathcal{T}}$]
The graph over the teams $\ccalT_m$, $m\in\ccalM$ is defined as $\ccalG_{\ccalT}=(\ccalV_{\ccalT},\ccalE_{\ccalT})$, where the set of nodes $\ccalV_{\ccalT}=\mathcal{M}$ is indexed by the teams $\ccalT_m$ and set of edges $\ccalE_{\ccalT}$ is defined as $\ccalE_{\ccalT}=\{(m,n)|\ccalT_m\cap\ccalT_n\neq\emptyset, \forall m,n\in\mathcal{M}, m\neq n\}$.
\end{defn}

Given the team membership graph $\ccalG_\ccalT$, we can also define the set
$\ccalN_{\ccalT_m}:=\left\{e\in\ccalV_{\ccalT}|(m,e)\in\ccalE_{\ccalT}\right\}$ that collects all neighboring teams of team $\ccalT_m $ in $\ccalG_{\ccalT}$. Since the robots have limited communication capabilities, we assume
that the robots in every subgroup $\ccalT_m$ can only communicate if all of them are simultaneously present at a common location $\bbv_j \in \ccalW$, hereafter called a communication point. We assume that there are
$R\geq 1$ available communication points in the workspace at locations $\bbv_j \in \ccalW$, where $j \in \ccalC \subset \ccalI$. Among those communication points, the ones that are specifically available to the robotic team $\ccalT_m$ are collected in a finite set $\ccalC_m\subseteq\ccalC$, where the sets $\ccalC_m$ are not necessarily disjoint. 
When all robots in a team $\ccalT_m$ have arrived at a communication location, we assume that communication happens and the robots leave to accomplish their tasks or communicate with other teams. This way, a dynamic robot communication network is constructed, defined as follows:

\begin{defn}[Communication Network $\ccalG_c(t)$]
The communication network among the robots is defined as a dynamic undirected graph $\ccalG_c(t)=(\mathcal{V}_c,\mathcal{E}_c(t))$, where the set of nodes $\ccalV_c$ is indexed by the robots, i.e., $\ccalV_c=\ccalN$, and $\ccalE_c(t)\subseteq\ccalV_c\times\ccalV_c$ is the set of communication links that emerge among robots in every team $\ccalT_m$, when they all meet at a common communication point $\bbv_j$, for some $j\in\ccalC_m$ simultaneously, i.e., $\ccalE_c(t)=\{(e,i), \forall \; i,e\in\mathcal{T}_m, \; \forall m\in\mathcal{M}~|~\bbx_i(t)=\bbx_e(t)=\bbv_j,~\text{for some}~ j\in\ccalC_m\}$.
\end{defn}

To ensure that information is continuously transmitted across the network of robots, we require that the communication graph $\ccalG_c(t)$ is \textit{connected over time infinitely often}, i.e., that all robots in every team $\ccalT_m$ meet infinitely often at a common communication point $\bbv_j$, $j\in\ccalC_m$, that does not need to be fixed over time. For this, it is necessary to assume that the graph of teams $\ccalG_\ccalT$ is connected. Specifically, if $\ccalG_\ccalT$ is connected, then information can be propagated intermittently across teams through robots that are common to these teams and, in this way, information can reach all robots in the network. Connectivity of $\ccalG_{\ccalT}$ and the fact that robots can be members of only a few teams means that information can be transferred over long distances, possibly to reach a remote user, without requiring that the robots leave their assigned regions of interest defined by their assigned tasks and communication points corresponding to the teams they belong to. Moreover, we assume that the teams $\ccalT_m$  are \textit{a priori} known and can be selected arbitrarily as long as the graph of
teams $\ccalG_\ccalT$ is connected. 

Intermittent connectivity of the communication network $\ccalG_c(t)$ can be captured by the global LTL formula 
\begin{equation}\label{eq:globalLTL}
\phi_{\text{com}}=\wedge_{ m\in\ccalM}\left(\square\Diamond\left(\vee_{j\in\ccalC_m}(\wedge_{ i\in\ccalT_m}\pi_{i}^{\bbv_j})\right)\right),
\end{equation}
specified over the set of atomic propositions $\{\{\pi_{i}^{\bbv_j}\}_{i=1}^N\}_{j\in\ccalC}$.
Composing $\phi_\text{com}$ with the local $\text{LTL}_{-\bigcirc}$ formulas $\phi_i$, yields the following global LTL statement
\begin{equation}\label{eq:taskandcom}
\phi=\left(\wedge_{i\in\ccalN}\phi_i\right)\wedge\phi_{\text{com}},
\end{equation} 
\noindent that captures the local tasks assigned to every robot and intermittent connectivity of the communication network $\ccalG_c$.

Given the $\text{wTS}_{i}$, for all robots $i\in\ccalN$, and the global LTL formula \eqref{eq:taskandcom}, the goal is to synthesize motion plans $\tau_i$,  for all $i\in\ccalN$, whose execution satisfies the global LTL formula \eqref{eq:taskandcom}. Typically, such motion plans are  infinite paths in $\text{wTS}_{i}$ \cite{clarke1999model}, i.e., infinite sequences of states in $\text{wTS}_i$, such that $\tau_{i}(1)=q_{i}^0$, $\tau_{i}(\kappa)\in\ccalQ_{i}$, and $(\tau_{i}(\kappa),\tau_{i}(\kappa+1))\in\rightarrow_{i}$, $\forall \kappa\in\mathbb{N}_+$. In this form, they cannot be manipulated in practice. This issue can be resolved by representing these plans in a prefix-suffix form \cite{vardi1986automata}, i.e., $\tau_{{i}}=\tau_{{i}}^{\text{pre}}\left[\tau_{{i}}^{\text{suf}}\right]^{\omega}$, where the prefix part $\tau_{{i}}^{\text{pre}}$ and suffix part $\tau_{{i}}^{\text{suf}}$ are both finite paths in $\text{wTS}_i$, for all robots $i\in\ccalN$. The prefix $\tau_{{i}}^{\text{pre}}$ is executed once and the suffix $\tau_{{i}}^{\text{suf}}$ is repeated indefinitely. The cost associated with a plan $\tau_{{i}}=\tau_{{i}}^{\text{pre}}\left[\tau_{{i}}^{\text{suf}}\right]^{\omega}$ is defined as
\begin{align}\label{eq:cost2}
%
J_p(\tau_i)=\alpha J(\tau_i^{\text{pre}})+(1-\alpha)J(\tau_i^{\text{suf}}),
\end{align}
where $ J(\tau_i^{\text{pre}})$ and $J(\tau_i^{\text{suf}})$ represent the cost of the prefix and the suffix part, respectively, and $\alpha\in[0,1]$ is a user-specified parameter. The cost $J(\tau_i^{\text{suf}})$ of the suffix part is defined as
\begin{equation}\label{eq:cost}
J(\tau_i^{\text{suf}})=\sum_{\kappa=1}^{|\tau_i^{\text{suf}}|}w_i(\tau_i^{\text{suf}}(\kappa),\tau_i^{\text{suf}}(\kappa+1)),
\end{equation}
where $|\tau_i^{\text{suf}}|$ stands for the number of states in the finite path $\tau_i^{\text{suf}}$, $\tau_i^{\text{suf}}(\kappa)$ denotes the $\kappa$-th state in $\tau_i^{\text{suf}}$, and $w_i$ are the weights defined in Definition \ref{defn:wTS}. The cost $J(\tau_i^{\text{pre}})$ of the prefix part is defined accordingly. In words, $J_p(\tau_i)$ captures the distance that robot $i$ needs to travel during a single execution of the prefix and suffix part weighted by a user-specified parameter $\alpha>0$.


 The problem that is addressed in this paper can be summarized as follows:
\begin{problem}\label{pr:pr1}
Consider any initial configuration of a network of $N$ mobile robots in their respective wTSs, and any partition of the network in $M$ subgroups $\ccalT_m$, $m\in\mathcal{M}$ so that the associated graph $\ccalG_{\ccalT}$ is connected.  Determine discrete motion plans $\tau_i$, i.e., sequences of states $q_i^{\bbv_j}\in\ccalQ_i$, in prefix-suffix structure, for all robots such that the LTL specification $\phi$ defined in \eqref{eq:taskandcom} is satisfied, i.e., (i) the local $\text{LTL}_{-\bigcirc}$ tasks $\phi_i$ are satisfied, for all $i\in\ccalN$, (ii) intermittent communication among robots captured by $\phi_{\text{com}}$ is ensured infinitely often, and (iii) the distance metric $\sum_{i\in\ccalN}J_p(\tau_i)$ is minimized.
\end{problem}

%
%

\begin{figure}[t]
    \centering
    \label{mobility}
     \includegraphics[width=0.5\linewidth]{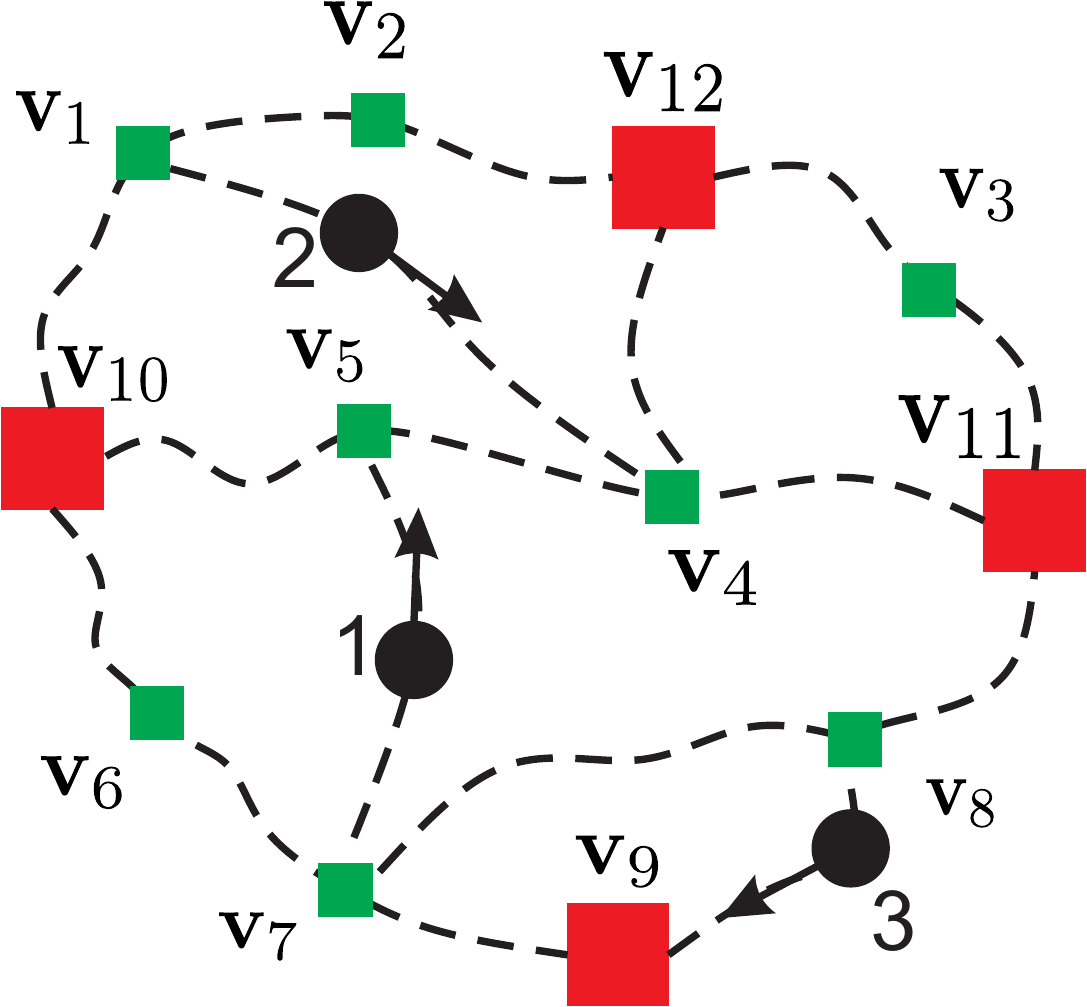}  \caption{A graphical illustration of the problem formulation. A network of $N=3$ robots (black dots) divided into $M=3$ teams is depicted. The robot teams are selected to be: $\ccalT_1=\{1,2\}$, $\ccalT_2=\{2,3\}$, and $\ccalT_3=\{3,1\}$. The set $\ccalI$ consists of locations represented by red and green squares. Red squares comprise set $\ccalC$ and represent communication points. Black dashed lines stand for paths in the workspace $\ccalW$ that connect locations $\bbv_e$ and $\bbv_j$. The sets of communications points for each team are defined as $\ccalC_1=\{\bbv_9,\bbv_{10}\}$, $\ccalC_2=\{\bbv_{10},\bbv_{11}\}$, and $\ccalC_3=\{\bbv_{12}\}$.}\label{fig:problem}
\end{figure} 

To solve Problem \ref{pr:pr1}, we propose a distributed algorithm that consists of two main parts. First, we design offline schedules of communication events for all robots, independently of their assigned tasks, that ensure intermittent communication among robots in every team infinitely often; see Section \ref{sec:commun}. These communication events depend on the structure of the graph $\ccalG_{\ccalT}$ and are not associated with specific locations in space. Then, in Section \ref{sec:integration} we design online discrete plans for the robots that satisfy their local tasks while ensuring that the robots in each team communicate as per the schedules defined in Section \ref{sec:commun}. The location of these communication events are selected so that the distance metric $\sum_{i\in\ccalN}J_p(\tau_i)$ is minimized. 


\section{Intermittent Communication Control}\label{sec:commun}

 In this section we construct infinite sequences of communication events (also called communication schedules) so that intermittent connectivity infinitely often as per \eqref{eq:globalLTL} is guaranteed. Construction of the communication schedules occurs offline i.e., before the robots are deployed in the workspace to satisfy the assigned $\text{LTL}_{-\bigcirc}$ tasks $\phi_i$, and requires that the robots are connected so that they can share information with each other. Then, in Section \ref{sec:integration}, these schedules are integrated online with task planning to synthesize discrete motion plans that ensure that the local tasks are satisfied, the network is intermittently connected as per the designed schedules, and the cost function  defined in Section \ref{sec:prob} is minimized.


Since every robot can be a member of more than one team, the objective in designing the proposed communication schedules is that no teams that share common robots communicate at the same time, as this would require that the shared robots are present at more than one possibly different communication points at the same time. We call such schedules conflict-free. To construct such conflict-free schedules of communication events we define a sequence $S$ of teams that determines the order in which the robots construct their schedules.
\begin{defn}[Sequence $S$]
\label{def:seqS}
The finite sequence $S$ is a sequence of teams defined as $S=\ccalT_n,\ccalT_m,\dots$. The sequence $S$ satisfies two requirements: (i) all teams $\ccalT_m$, $m\in\mathcal{M}$ appear in $S$; and (ii) consecutive teams $\ccalT_n$ and $\ccalT_m$ that appear in $S$ are neighboring nodes in the graph $\ccalG_{\ccalT}$, i.e., $m\in\ccalN_{\ccalT_n}:=\left\{e\in\ccalV_{\ccalT}|(n,e)\in\ccalE_{\ccalT}\right\}$.
\end{defn}

In what follows, we assume that the sequence $S$ is user-defined and known by all robots. Moreover, we denote by $S(k)$ the $k$-th team in $S$, $\forall k\in\{1,\dots,|S|\}$ where $|S|$ stands for the length of $S$. Using the sequence $S$ we construct communication schedules $\texttt{sched}_i$ for all robots $i$ that determine the order in which those robots participate in communication events for teams $\ccalT_m$, $\forall m\in\ccalM_i$ and are defined as follows: 
\begin{defn}[Schedule of Communication Events]\label{defn:sched}
The schedule $\texttt{sched}_i$ of communication events of robot $i$ is defined as an infinite repetition of the finite sequence 
\begin{align}\label{eq:si} 
 s_i=&X,\dots,X,\ccalM_i(1),X,\dots,X,\ccalM_i(2),X,\dots,X,\nonumber \\ &\ccalM_i(\left|\ccalM_i\right|),X,\dots, X,
\end{align} 
i.e., $\texttt{sched}_i=s_i,s_i,\dots=s_i^{\omega}$, where $\omega$ stands for the infinite repetition of $s_i$.
\end{defn}

In \eqref{eq:si}, $\ccalM_i(e)$, $e\in\{1,\dots,|\ccalM_i|\}$ stands for the $e$-th entry of $\ccalM_i$ and represents a communication event for team with index $\ccalM_i(e)$, and the discrete states $X$ indicate that there is no communication event for robot $i$. The length of the sequence $s_i$ is $\ell=\max\left\{d_{\ccalT_m}\right\}_{m=1}^M+1$ for all $i\in\ccalN$, where $d_{\ccalT_m}$ is the degree of node $m\in\ccalV_{\ccalT}$. It is shown in Proposition \ref{prop:ell} that this length $\ell$ is sufficient for the construction of conflict-free communication schedules as per the algorithm described bellow. The schedule $\texttt{sched}_i$ defines the order in which robot $i$ participates in the communication events for the teams $m\in\ccalM_i$, for all robots $i\in\ccalN$. Specifically, at a discrete time step $z\in\mathbb{N}_+$, robot $i$ either communicates with all robots that belong to team $\ccalT_m$, for $m\in\ccalM_i$ if $\texttt{sched}_i(z)=m$, or does not need to participate in any communication event if $\texttt{sched}_i(z)=X$.


In what follows we present a distributed process that relies on two rules that the robots execute in order to construct the schedules $\texttt{sched}_i$. These schedules are constructed sequentially across the teams $\ccalT_m$, $m\in\ccalM$, in an order that is determined by the sequence $S$. In other words, robots in team $S(k)$ will construct their respective schedules, only if all robots in team $S(k-1)$ have already designed their schedules. Assume that according to the sequence $S$, robots in team $S(k)=\ccalT_m$, for some $k\geq 1$ are due to construct their schedules. By construction of the sequence $S$, consecutive teams in $S$ are always neighboring teams, which means that there exists a team $\ccalT_n$ with $n\in \ccalN_{\ccalT_m}$ such that $S(k-1)=\ccalT_n$ and $\ccalT_m\cap \ccalT_n \neq \emptyset$. Consequently, there exist also robots $j \in \ccalT_m \cap \ccalT_n$ that previously constructed their sequences $s_{j}$. These robots $j$ never re-construct their schedules. Instead, one of the robots $j\in S(k)\cap S(k-1)$ is tasked with providing information to the other robots $i\in S(k)=\ccalT_m$ that is necessary to construct their sequences $s_i$.

Specifically, this robot $j\in S(k)\cap S(k-1)$ first notifies the robots in team $S(k)=\ccalT_m$ that it is their turn to construct their communication schedules.\footnote{Note that if the teams in $S$ were not necessarily neighboring teams, then robot $j\in S(k-1)=\ccalT_n$ would have to know who the members of team $ S(k)=\ccalT_m$, $m\notin\ccalM_j$, are in order to notify them that it is their turn to construct the communication schedules. Due to the fact that $S$ connects neighboring teams, every robot $j$ needs to know only the structure of teams $\ccalT_m$, $m\in\ccalM_j$.} Second, robot $j$ transmits to robots $i\in \ccalT_m$ all sequences $s_b$ that were have been constructed so far by the robots in teams $S(1),\dots,S(k-1)$. Among all those sequences $s_b$, robots $i\in\ccalT_m$ use only the sequences of robots $b\in\ccalN_i$ to construct their sequences $s_i$.\footnote{Note that robot $j$ is not aware of the sets $\ccalN_i$ and, therefore, it transmits all the sequences $s_b$ that have already been constructed to robots $i\in\ccalT_m$.}
As a result, all robots $i\in\ccalT_m$ that have not constructed $s_i$ yet, are aware of the indices $n_b^{\ccalT_g}$ that point to entries in $s_b$ associated with some communication events $g$. These indices satisfy $s_b(n_b^{\ccalT_g})=g$, $b\in\ccalN_i$.\footnote{Note that the discrete time instants at which the communication event $g\in\ccalM_i$ will take place are $n_i^{\ccalT_g}+z \ell$, where $z\in\mathbb{N}$, by definition of $\texttt{sched}_i$.} Notice that this means that robots $i\in\ccalT_m$ are also aware of the indices $n_b^{\ccalT_m}$. Using this information, every robot $i\in\ccalT_m$ constructs the sequence $s_{i}$ based on the following two rules that determine the indices $n_i^{\ccalT_g}$ that point to entries in $s_i$ where the communication event $g$ will be placed, i.e., $s_i(n_i^{\ccalT_g})=g$, for all $g\in\ccalM_i$. 

\begin{enumerate}
\item \textit{First rule:} Let $n_{i}^{\ccalT_g}$ denote the index of the entry at which the communication event $g\in\ccalM_i$ will be placed into $s_i$. If there exists a robot $b\in\ccalN_i$ that has selected $n_{b}^{\ccalT_g}$ so that $s_b(n_{b}^{\ccalT_g})=g$, then $n_{i}^{\ccalT_g}=n_{b}^{\ccalT_g}$. In this way,  all robots $b\in\ccalT_g$, including robot $i\in\ccalT_m\cap\ccalT_g$ will select the same index $n_{b}^{\ccalT_g}$ and will participate in the same communication event $g$ at the same discrete time instant; see line \ref{schedule:line4}, Alg. \ref{alg:schedule}.
%


\item \textit{Second rule:} If there do not exist robots $b\in\ccalN_i$ that have selected indices $n_b^{\ccalT_g}$, for communication event $g\in\ccalM_i$, then the communication event $g$ can be placed  at any available entry $n_{i}^{\ccalT_g}$ of $s_i$ that satisfies the following requirement. The entry $n_{i}^{\ccalT_g}$ in all sequences $s_j$ of robots $j\in\ccalN_{i}$ that have already been constructed should not contain communication events $h$ such that $h\in\ccalN_{\ccalT_g}$; see line \ref{schedule:line6}, Alg. \ref{alg:schedule}. 
\end{enumerate}
Note that the index $n_i^{\ccalT_m}$ will always be determined by the first rule, since robot $j\in S(k)\cap S(k-1)$ has already constructed its sequence $s_j$ by placing the event $m$ at an entry of $s_j$ with index $n_j^{\ccalT_m}$.
To highlight the role of the second rule assume that $h\in\ccalN_{\ccalT_g}$. Then, this means that there exists at least one robot $r\in\ccalT_h\cap\ccalT_g$. Notice that without the second rule, at a subsequent iteration of this procedure, robot $r\in\ccalT_h\cap\ccalT_g$ would have to place communication events for teams $\ccalT_g$ and $\ccalT_h$ at a common entry of $s_r$, i.e., $n_{r}^{\ccalT_g}=n_{r}^{\ccalT_h}$, due to the first rule and, therefore, a conflicting communication event in schedule $\texttt{sched}_r$ would occur.  In all the remaining entries of $s_i$, $X$'s are placed; see line \ref{schedule:line7}, Alg. \ref{alg:schedule}. By construction of $s_i$, there are $\ell-|\ccalM_i|$ $X$'s in $s_i$. 

Once all robots $i$ in team $S(k)$ have constructed the sequences $s_i$, a robot $j\in S(k)\cap S(k+1)$ will notify all robots in team $S(k+1)$ that it is their turn to compute their respective schedules. The procedure is repeated sequentially over the teams in $S$ until all robots have computed their respective schedules of meeting events. This process is summarized in Algorithm \ref{alg:schedule} and it is also illustrated in Example \ref{ex:ex1}.

\begin{algorithm}[t]
\caption{Distributed construction of sequence $s_i$, $i\in\ccalT_m$}
\LinesNumbered
\label{alg:schedule}
\KwIn {Already constructed sequences $s_{b}$, $\forall b\in\ccalN_i$.}
\KwOut {Schedule of meeting events: $\texttt{sched}_i=[s_i]^{\omega}$}
Construct an empty finite sequence $s_i$ of length $\ell$. \;
\For {$g\in\ccalM_i$}{\label{schedule:line2}
	\If {there exist constructed sequences $s_b$, $b\in\ccalT_g$}{\label{schedule:line3}
	   $s_i(n_{i}^{\ccalT_g}):=g$, where $n_{i}^{\ccalT_g}:=n_{b}^{\ccalT_g},~\forall b\in\ccalT_g$ \Comment*[r]{First rule}}\label{schedule:line4}
	\Else{\label{schedule:line5}
	 Choose an available $n_{i}^{\ccalT_g}\in\{1,\dots,\ell\}$ such that it holds either 	$s_j(n_{i}^{\ccalT_g}):=X$, or $s_j(n_{i}^{\ccalT_g}):=h$ with $h\notin\ccalN_{\ccalT_g}$, $~\forall j\in\ccalN_{i}$. Then set $s_i(n_{i}^{\ccalT_g}):=g$. \Comment*[r]{Second rule}}}\label{schedule:line6}
    Put $X$ in the remaining entries\;\label{schedule:line7}
\end{algorithm}
 
\begin{ex}[Algorithm \ref{alg:schedule}]
To illustrate Algorithm \ref{alg:schedule}, consider the network of $N=3$ robots shown in Figure \ref{fig:problem}, where the teams of robots are designed as $\ccalT_1=\{1,2\}$, $\ccalT_2=\{2,3\}$, and $\ccalT_3=\{3,1\}$. Let the sequence $S$ be $S=\ccalT_1,\ccalT_2,\ccalT_3$. Hence, initially the robots 1 and 2 in team $\ccalT_1$ coordinate to construct their respective sequences $s_i$. Assume that initially robot $1$ constructs the sequence $s_1$ of length equal to $\ell=\text{max}\left\{d_{\ccalT_m}\right\}_{m=1}^3+1=3$. Robot $1$ belongs to teams $\ccalT_1$ and $\ccalT_2$ and it arbitrarily constructs $s_1$ as follows: $s_1=1,3,X$. Then the sequence $s_1$ is transmitted to robot $2$ that belongs to teams $\ccalT_1$ and $\ccalT_2$. Now robot $2$ is responsible for constructing the sequence $s_2$. To construct $s_2$, according to the first rule, team $\ccalT_1$ is placed at the first entry of $s_2$, i.e., $n_{2}^{\ccalT_1}=n_{1}^{\ccalT_1}=1$. Next, the index $n_{2}^{\ccalT_2}$ is determined by the second rule. Specifically, notice that among the two available entries in $s_2$ for team $\ccalT_3$ the entry $n_{2}^{\ccalT_3}=2$ is invalid, since robot $1\in\ccalT_1$ has already constructed its sequence $s_1$ so that $n_{1}^{\ccalT_3}=2$ and for teams $\ccalT_3$ and $\ccalT_2$ it holds that $3\in\ccalN_{\ccalT_2}$. Therefore, robot $2$ selects $n_{2}^{\ccalT_2}=2$ and constructs the sequence $s_2=1,X,2$. At the next iteration of Algorithm \ref{alg:schedule} the robots 2 and 3 in team $\ccalT_2$ coordinate to construct their sequences $s_i$. Robot $2$ has already constructed the sequence $s_2$ at the previous iteration and it transmits its constructed sequence $s_2$ and the previously constructed sequence $s_1$ to robot $3$. Thus robot $3$ has now access to all already constructed sequences $s_e$, for $e\in\ccalN_3=\{1,2\}$. Robot $3$ constructs $s_3=X,3,2$ using the first rule. Finally, the robots in the third team $\ccalT_3=\{3,1\}$ have already constructed their finite paths at previous iterations.
\label{ex:ex1}
\end{ex}

In the following proposition we show that Algorithm \ref{alg:schedule} can always construct sequences $s_i$ if the length $\ell$ of $s_i$ is selected as $\ell=\text{max}\left\{d_{\ccalT_m}\right\}_{m=1}^M+1$. 

\begin{prop}\label{prop:ell}
Algorithm \ref{alg:schedule} can always construct sequences $s_i$, for all $i\in\ccalN$, if the length $\ell$ of $s_i$ is selected as $\ell=\text{max}\left\{d_{\ccalT_m}\right\}_{m=1}^M+1$.

\end{prop}

\begin{proof}
The proof is based on contradiction. Assume that a robot $i$ requires a sequence $s_i$ of length greater than $\ell=\max\left\{d_{\ccalT_e}\right\}_{e=1}^{M}+1$ when Algorithm \ref{alg:schedule} is applied. This means that there is team $\ccalT_m$, $m\in\ccalM_i$, which cannot be placed at any of the first $\ell$ entries of $s_i$. By construction of Algorithm \ref{alg:schedule}, this means that the team $\ccalT_m$ has at least $\ell$ neighbors in graph $\ccalG_{\ccalT}$, i.e.,  $d_{\ccalT_m}\geq\ell$, which can never happen, which completes the proof.
\end{proof}

\begin{rem}[Repeated teams in $S$ and initialization]
Due to the requirement that consecutive teams in $S$ need to be neighbors in $\ccalG_{\ccalT}$, it is possible that a team $\ccalT_m$ may appear more than once in $S$, depending on the structure of the graph $\ccalG_{\ccalT}$. In this case, robots $i\in\ccalT_m$ construct the sequences $s_i$ only the first time that team $\ccalT_m$ appears in $S$. Also, at the first iteration of Algorithm \ref{alg:schedule}, robots of team $S(1)$ have to construct their sequences $s_i$, $i\in S(1)$. In this case, a randomly selected robot $j\in S(1)$ creates arbitrarily its sequence $s_j$ by placing the teams $m\in\ccalM_j$  at the $n_{j}^{\ccalT_m}$-th entry of $s_j$. Then the procedure described in Algorithm \ref{alg:schedule} follows. 
\end{rem}

\begin{rem}[Discrete states $X$]\label{rem:X}
In the schedules $\texttt{sched}_i$, defined in Definition \ref{defn:sched} and constructed using Algorithm \ref{alg:schedule}, the states $X$ indicate that no communication events occur for robot $i$ at the corresponding discrete time instants. These states are used to synchronize the communication events over the discrete time instants $c\in\mathbb{N}_+$, i.e., to ensure that the discrete time instant $z$ at which communication happens for team $\ccalT_m$, $m\in\ccalM$, is the same for all robots $i\in\ccalT_m$; see also Example \ref{ex:ex1}. Nevertheless, as it will be shown in Theorem \ref{thm:com}, in Section \ref{sec:analysis}, it is the order of communication events in $\texttt{sched}_i$ that is critical to ensure intermittent communication, not the time instants that they take place. This is due to a communication policy proposed in \ref{sec:execution}. 
\end{rem}

\section{Integrated Task Planning and Intermittent Communication Control}\label{sec:integration}

In this section, we propose a distributed and online algorithm to synthesize motion plans for all robots $i$ so that the global LTL formula \eqref{eq:globalLTL} is satisfied, i.e., the assigned local $\text{LTL}_{-\bigcirc}$  tasks are accomplished, and the network is intermittently connected. These plans are generated iteratively and have the following prefix-suffix structure
\begin{equation}\label{eq:taui}
\tau_i^{n_i}=\texttt{path}_i^{0}|\texttt{path}_i^{1}|\dots|[\texttt{path}_i^{n_i}]^{\omega},
\end{equation}
where $n_i\in\mathbb{N}$ is the iteration index associated with robot $i$,  $\texttt{path}_i^{n_i}$ is a finite sequence of states in $\text{wTS}_i$, $|$ denotes the concatenation of discrete paths $\texttt{path}_i^{n_i}$, and $\omega$ denotes the infinite repetition. 
Each path $\texttt{path}_i^{n_i}$ is constructed so that (i) execution of $\texttt{path}_i^{n_i}$, for a every given $n_i$ ensures that robot $i$ will communicate exactly once with all teams $\ccalT_m$, $m\in\ccalM_i$ in an order that respects the schedules $\texttt{sched}_i$ designed in Section \ref{sec:commun}, and (ii) execution of $\tau_i^{n_i}$ guarantees that the assigned local $\text{LTL}_{-\bigcirc}$  tasks $\phi_i$ are satisfied. 
In Section \ref{sec:init}, we discuss the distributed construction of the initial paths $\texttt{path}_i^{0}$ given the communication schedules $\texttt{sched}_i$. In Section \ref{sec:pathni}, we present the  distributed construction of all subsequent paths $\texttt{path}_i^{n_i}$ that occurs online as the robots navigate the worskpace. 

\subsection{Construction of Initial Paths}\label{sec:init} 
Once robot $i$ constructs its schedule $\texttt{sched}_i$, it locally designs the initial path $\texttt{path}_i^{0}$. To do this, feasible initial communication points for all teams $\ccalT_m$, $m\in\ccalM$, need to be selected first, that do not violate the local tasks $\phi_i$. These can be found by exhaustively searching through the set of possible combinations of communication points for all teams.
%
%
Specifically, let $\texttt{comb}_b$ denote any candidate combination of communication points that can be assigned to all teams $\ccalT_m$, $m\in\ccalM$, where $b\in\set{1,\dots,\prod_{m\in\ccalM}|\ccalC_m|}$.
Given the communication points $\bbv_j$, $j\in\ccalC_m$, in the candidate combination $\texttt{comb}_b$, every robot constructs the NBA $B_{i}$ that corresponds to the following LTL formula
\begin{equation}\label{eq:psi}
\psi_i=\underbrace{\phi_i}_{\text{task}}\wedge \underbrace{\phi_{\text{com},i}}_{\text{communication}},
\end{equation}
where
\begin{equation}\label{eq:phicomi}
 \phi_{\text{com},i}=\wedge_{m\in\ccalM_i}(\square\Diamond\bbv_{j\in\ccalC_m}),
\end{equation}
In words, the LTL formula  $\phi_{\text{com},i}$ requires robot $i$ to visit infinitely often the candidate communication points $\bbv_j$, $j\in\ccalC_m$, of all teams $\ccalT_m$, $m\in\ccalM_i$, that are specified in $\texttt{comb}_b$. Then, given the $\text{wTS}_i$ and the NBA $B_{i}$, every robot can synthesize a motion plan $\tilde{\tau}_i^0\models\psi_{i}$, if it exists, which will be used to construct the initial path $\texttt{path}_i^{0}$. This process is repeated for all $b\in\set{1,\dots,\prod_{m\in\ccalM}|\ccalC_m|}$ until feasible plans $\tilde{\tau}_i^0\models\psi_{i}$ can be constructed for all robots~$i\in\ccalN$. Later, in Lemma \ref{lem:init}, we show that the robots can search locally over the combinations $\texttt{comb}_b$ reducing in this way the computational cost of finding a feasible plan $\tilde{\tau}_i^0$.


Specifically, given candidate initial communication points for all teams $\ccalT_m$, $m\in\ccalM_i$, the motion plan $\tilde{\tau}_i^0$ can be constructed by checking the non-emptiness of the language of the \textit{Product B$\ddot{\text{u}}$chi Automaton} (PBA) $P_{i}=\text{wTS}_{i}\otimes B_{i}$, defined as~follows~\cite{baier2008principles}:

\begin{defn}[Product B$\ddot{\text{u}}$chi Automaton]\label{def:pba}
Given the weighted transition system $\text{wTS}_{i}=\left(\ccalQ_{i}, q_{i}^0,\rightarrow_{i},w_{P_i},\mathcal{AP},L_{i}\right)$ and the NBA $B_{i}=\left(\ccalQ_{B_{i}}, \ccalQ_{B_{i}}^0,2^{\mathcal{AP}},\rightarrow_{B_{i}},\mathcal{F}_{B_{i}}\right)$, the \textit{Product B$\ddot{\text{u}}$chi Automaton} $P_{i}=\text{wTS}_{i}\otimes B_{i}$ is a tuple $\left(\ccalQ_{P_{i}}, \ccalQ_{{P_{i}}}^0,\longrightarrow_{P_{i}},w_{P_i},\ccalF_{P_{i}}\right)$ where (a) $\ccalQ_{P_{i}}=\ccalQ_{{i}}\times\ccalQ_{B_{i}}$ is the set of states; (b) $\ccalQ_{P_{i}}^0=q_{i}^0\times\ccalQ_{B_{i}}^0$ is a set of initial states; (c) $\longrightarrow_{P_{i}}\subseteq\ccalQ_{P_{i}}\times\ccalQ_{P_{i}}$ is the transition relation. Transition $(q_P,q_P')\in\rightarrow_{P_i}$, where $q_P=(q_i^{\bbv_j},q_B)\in\ccalQ_{P_i}$ and $q_P'=(q_i^{\bbv_e},q_B')\in\ccalQ_{P_i}$, exists if $(q_i^{\bbv_j},q_i^{\bbv_e})\in\rightarrow_i$ and $(q_B,L_i(q_i^{\bbv_j}),q_B')\in\rightarrow_B$;
(d) $w_{P_i}:\ccalQ_{P_{i}}\times\ccalQ_{P_{i}}\rightarrow\mathbb{R}_+$ is the weight function, defined as: $w_{P_i}((q_i^{\bbv_j},q_B),(q_i^{\bbv_e},q_B'))=w_i(q_i^{\bbv_j},q_i^{\bbv_e})$; and (e) $\ccalF_{P_{i}}=\ccalQ_{{i}}\times\ccalF_{B_{i}}$ is a set of accepting/final states. 
\end{defn}

More precisely, a motion plan $\tilde{\tau}_i^0$ that satisfies $\psi_i$ can be derived using graph search techniques on $P_i$, which can be viewed as a weighted graph $\ccalG_{P_{i}}=\left\{\ccalV_{P_{i}},\ccalE_{P_{i}},w_{P_i}\right\}$, where $\ccalV_{P_{i}}=\ccalQ_{P_{i}}$, the set of edges $\ccalE_{P_{i}}$ is determined by the transition relation $\longrightarrow_{P_{i}}$, and the weight function $w_{P_i}$ is defined in Definition \ref{def:pba}; see e.g., \cite{ulusoy2013optimality,ulusoy2014optimal,kantaros2017Csampling,kantaros2017Dsampling,smith2011optimal,guo2015multi}. Then, a path from an initial state to an accepting state in $\ccalG_{P_{i}}$ (the prefix path) followed by a cycle around this accepting state (the suffix path), which is repeated indefinitely, results in an accepting run of the PBA that has the following prefix-suffix structure 
\begin{align}\label{eq:rho0}
 \rho_{P_{i}}^0=&\rho_{P_{i}}^{\text{pre},0}\left[\rho_{P_{i}}^{\text{suf},0}\right]^{\omega}=\underbrace{(q_{\text{wTS}_{i}}^0,q_{B_{i}}^0)}_{\in\ccalQ_{P_{i}}^0}(q_{\text{wTS}_{i}}^1,q_{B_{i}}^1)\dots\underbrace{(q_{\text{wTS}_{i}}^F,q_{B_{i}}^F)}_{=q_{P_i}^F\in\ccalF_{P_{i}}}\nonumber\\&\left[(q_{\text{wTS}_{i}}^F, q_{B_{i}}^F)\dots(q_{\text{wTS}_{i}}^L,q_{B_{i}}^L)\right]^\omega,
\end{align}
where with slight abuse of notation, $q_{\text{wTS}_{i}}^\beta$ and $q_{B_{i}}^\beta$ denote a state of $\text{wTS}_i$ and $B_i$, respectively, for all $\beta\in\set{0,\dots,F,\dots,L}$. The projection of $\rho_{P_{i}}^0$ onto the state-space of $\text{wTS}_{i}$, denoted by $\Pi|_{\text{wTS}_{i}}\rho_{P_{i}}^0$, results in the desired prefix-suffix motion plan 
\begin{align} \tilde{\tau}_i^0&=\Pi|_{\text{wTS}_{i}}\rho_{P_{i}}^0=\tilde{\tau}_i^{\text{pre},0}\left[\tilde{\tau}_i^{\text{suf},0}\right]^{\omega}\nonumber\\&=\left[q_{\text{wTS}_{i}}^0\dots q_{\text{wTS}_{i}}^F\right]\left[q_{\text{wTS}_i}^F\dots q_{\text{wTS}_{i}}^L\right]^\omega,
\end{align} 
that satisfies $\psi_{i}$ provided feasible initial communication points have been selected \cite{vardi1986automata}. To reduce the computational cost of synthesizing $\tilde{\tau}_i^0$, we only require a feasible plan $\tilde{\tau}_i^0 $ and not the optimal one that minimizes \eqref{eq:cost2}, especially since subsequent paths $\texttt{path}_i^{n_i}$ will get optimized online. 


%

Given the motion plans $\tilde{\tau}_i^0=\tilde{\tau}_i^{\text{pre},0}[\tilde{\tau}_i^{\text{suf},0}]^{\omega}$, we design the discrete paths $\texttt{path}_i^0$ as follows. 
First, we initialize $\texttt{path}_i^0$ as $\texttt{path}_i^0=\tilde{\tau}_i^{\text{pre},0}|\tilde{\tau}_i^{\text{suf},0}$. Recall that all paths $\texttt{path}_i^0$ are designed so that if executed, then robot $i$ will communicate once with all teams $\ccalT_m$, $m\in\ccalM_i$, in an order that respects the schedules $\texttt{sched}_i$. Therefore, the state $q_i^{\bbv_j}$ corresponding to the candidate communication point $\bbv_j$, $j\in\ccalC_m$, appears at least once in the suffix part of $\tilde{\tau}_i^0$, by definition of $\psi_i$, for all $m\in\ccalM_i$. 
However, these communication states may not appear in $\texttt{path}_i^0=\tilde{\tau}_i^{\text{pre},0}|\tilde{\tau}_i^{\text{suf},0}$ in an order that respects the schedules $\texttt{sched}_i$, as this is not required by the LTL formula $\psi_i$ in \eqref{eq:psi}. Therefore, we append at the end of $\texttt{path}_i^0$ the suffix part $\tilde{\tau}_{{i}}^{\text{suf},0}$ enough times so that $\texttt{path}_i^0=\tilde{\tau}_i^{\text{\text{pre}},0} | \tilde{\tau}_i^{\text{suf},0} | \dots | \tilde{\tau}_i^{\text{suf},0}$ respects
 the schedule $\texttt{sched}_i$, i.e., there exists a sequence of indices $\kappa_i^m$ that point to entries in 
$\texttt{path}_i^0$ corresponding to states $q_i^{\bbv_j}$ with $\bbv_j$, $j\in\ccalC_m$, that satisfy $\kappa_i^m<\kappa_i^h$, if the communication event for team $\ccalT_m$ appears before the communication event for team $\ccalT_h$ in  
$\texttt{sched}_i$, for all teams $\ccalT_m,~\ccalT_h$, $m,h\in\ccalM_i$; see also Example \ref{ex:path0}. Note that since the state $q_i^{\bbv_j}$, $j\in\ccalC_m$, appears at least once in the suffix part of $\tilde{\tau}_i^0$, for all $m\in\ccalM_i$, the suffix part $\tilde{\tau}_i^{\text{suf},0}$ will be appended to $\texttt{path}_i^0$ at most $|\ccalM_i|-1$ times. With slight abuse of notation, the initial path $\tau_i^{0}$ in \eqref{eq:taui} is defined using only $\texttt{path}_i^0$ as follows:
\begin{equation}\label{eq:taui0}
\tau_i^{0} =\tilde{\tau}_i^{\text{pre},0}[  \tilde{\tau}_i^{\text{suf},0}|\dots| \tilde{\tau}_i^{\text{suf},0}]^{\omega}
\end{equation}

%
In what follows, we show that to find a feasible initial combination of communication points $\texttt{comb}_b$ that is needed to determine initial plans $\tilde{\tau}_i^0$, the robots can search locally in the set of $\prod_{m\in\ccalM}|\ccalC_m|$ possible combinations of communication points by solving at most $\prod_{m\in\ccalM_i}|\ccalC_m|$ control synthesis problems each, instead of $\prod_{m\in\ccalM}|\ccalC_m|$.
To see this, observe that, for any robot $i\in \ccalN$, there exist multiple combinations $\texttt{comb}_b$ that share the same communication points for all teams $\ccalT_m$, $m\in \ccalM_i$, and only differ in the communication points for teams $\ccalT_m$, $m\in \ccalM\setminus\ccalM_i$. All these combinations, correspond to the same formula $\psi_i$, which means that that robot $i$ needs to solve a single control synthesis problem to determine if they are feasible. Motivated by this observation, in the following lemma, we show that if every robot $i\in\ccalN$ solves locally at most $\prod_{m\in\ccalM_i}|\ccalC_m|$ control synthesis problems, then all combinations $\texttt{comb}_b$ will be exhaustively explored. By combining the feasible local combinations of communication points $\texttt{comb}_{b_i}^i$ that can be assigned to teams $\ccalT_m$, $m\in\ccalM_i$, where $b_i\in\{1,..., \prod_{m\in \ccalM_i}|\ccalC_m|\}$, that are identified by all robots $i$, it it easy to obtain feasible global combinations $\texttt{comb}_b$. Note that, in general, it holds that $\prod_{m\in\ccalM_i}|\ccalC_m|\leq \prod_{m\in\ccalM}|\ccalC_m|$, where the equality holds if $\ccalM_i=\ccalM$ or if $|\ccalC_m|=1$, for all~$m\in\ccalM\setminus\ccalM_i$. Moreover, $\prod_{m\in\ccalM_i}|\ccalC_m|$ is smaller for sparse graphs $\ccalG_{\ccalT}$, given a fixed number of teams and fixed~sets~$\ccalC_m$.

\begin{lem}[Complexity of initialization]\label{lem:init}
Let $\texttt{comb}_{b_i}^i$ with $b_i\in\set{1,\dots,\prod_{m\in\ccalM_i}|\ccalC_m|}$ denote a combination of communication points that can be assigned to teams~$\ccalT_m$,~$m\in\ccalM_i$. Moreover, assume that every robot $i\in\ccalN$ solves $\prod_{m\in\ccalM_i}|\ccalC_m|$ control synthesis problems using the LTL formula \eqref{eq:psi}, one for every combination $\texttt{comb}_{b_i}^i$. Then, the robots can collectively detect any feasible combination of communication points $\texttt{comb}_{b}$, $b\in\set{1,\dots,\prod_{m\in\ccalM}|\ccalC_m|}$, if it exists, that can be assigned to all teams~$\ccalT_m$,~$m\in\ccalM$.
\end{lem}
\begin{proof}
In what follows, we show by contradiction that under this local construction of $\texttt{comb}_{b}$, the robots can detect all feasible combinations $\texttt{comb}_{b}$. 
Assume that there exists a feasible combination $\texttt{comb}_{b}$, that cannot be detected if all robots solve their respective $\prod_{m\in\ccalM_i}|\ccalC_m|$ control synthesis problems. Also, let $\Pi|_{\ccalM_i}\texttt{comb}_{b}$ denote the combination of communication points in $\texttt{comb}_{b}$ that correspond to all teams $\ccalT_m$, $m\in\ccalM_i$.
Since $\texttt{comb}_{b}$ cannot be detected by the robots, this means that there exists at least one robot $i$ that either could not find a feasible solution to the control synthesis problem that corresponds to the combination $\Pi|_{\ccalM_i}\texttt{comb}_{b}$ or did not consider the combination $\Pi|_{\ccalM_i}\texttt{comb}_{b}$. The first case contradicts the assumption that $\texttt{comb}_{b}$ is a feasible combination of communication points that can be assigned to all teams $\ccalT_m$, $m\in\ccalM$, while the second case contradicts the assumption that every robot $i\in\ccalN$ searches over all combinations $\texttt{comb}_{b_i}^i$, completing the proof.
\end{proof}

\begin{ex}[Construction of $\texttt{path}_i^0$]
Consider a robot $i$ with $\ccalM_i=\set{2,3,4,5}$ and communication schedule $\texttt{sched}_i=[2,3,X,4,5]^\omega$. Consider also the motion plan $\tilde{\tau}_i^0=\tilde{\tau}_i^{\text{pre},0}[\tilde{\tau}_i^{\text{suf},0}]^{\omega}=q_i^{\bbv_1}q_i^{\bbv_6}q_i^{\bbv_4}q_i^{\bbv_5}q_i^{\bbv_2}q_i^{\bbv_3}[q_i^{\bbv_3}q_i^{\bbv_5}q_i^{\bbv_4}q_i^{\bbv_6}q_i^{\bbv_2}]^\omega$, where $\bbv_2,~\bbv_3,~\bbv_4$ are the candidate communication points for teams $\ccalT_2,~\ccalT_3,~\ccalT_4$, respectively. The path $\texttt{path}_i^0$ is initialized as $\texttt{path}_i^0=\tilde{\tau}_i^{\text{pre},0}|\tilde{\tau}_i^{\text{suf},0}$. To ensure the existence of indices $\kappa_i^m$ in $\texttt{path}_i^0$ for all teams $\ccalT_m$, $m\in\ccalM_i$, that respect the schedule $\texttt{sched}_i$, the suffix part needs to be appended to  $\texttt{path}_i^0$ once more, i.e., $\texttt{path}_i^0=q_i^{\bbv_1}q_i^{\bbv_6}q_i^{\bbv_4}q_i^{\bbv_5}q_i^{\bbv_2}q_i^{\bbv_3}[q_i^{\bbv_3}q_i^{\bbv_5}q_i^{\bbv_4}q_i^{\bbv_6}q_i^{\bbv_2}][q_i^{\bbv_3}q_i^{\bbv_5}q_i^{\bbv_4}q_i^{\bbv_6}q_i^{\bbv_2}]$, where the sequence of states in brackets stands for the suffix part $\tau_{{i}}^{\text{suf},0}$. Observe that in $\texttt{path}_i^0$ there exists indices $\kappa_i^2=5$, $\kappa_i^3=6$, $\kappa_i^4=9$ and $\kappa_i^5=13$, so that $\kappa_i^2<\kappa_i^3<\kappa_i^4<\kappa_i^5$ as~dictated~by~$\texttt{sched}_i$.
\label{ex:path0}
\end{ex}

\begin{rem}[Initialization]\label{rem:init}
Note that there are cases where feasible initial communication points can be easily identified by inspection, e.g., if there exists a communication point $\bbv_j$, $j\in\ccalC_m$, that (i) does not appear in the atomic propositions $\pi_i^{\bbv_e}$ that capture the tasks $\phi_i$ assigned to robots $i\in\ccalT_m$, and (ii) is directly connected to all locations $\bbv_e$, $e\in\ccalI$, that robots $i\in\ccalT_m$ should visit to accomplish their tasks, i.e., the atomic propositions $\pi_i^{\bbv_e}$ appear in the tasks $\phi_i$, $i\in\ccalT_m$. Then, $\bbv_j$, $j\in\ccalC_m$, is a feasible communication point for team $\ccalT_m$, since it does not violate the tasks $\phi_i$ for all $i\in\ccalT_m$ and it does not affect the communication points the other teams can select due to (i). Also, due to (ii) robots $i\in\ccalT_m$ can visit $\bbv_j$ directly from any location $\bbv_e$ without passing through locations that may violate $\phi_i$.  Finally, if the negation operator does not appear in the tasks $\phi_i$ of all robots $i\in\ccalT_m$, then any communication point $\bbv_j$, $j\in\ccalC_m$, assigned to team $\ccalT_m$ is feasible.
\end{rem}



\begin{rem}[Formula $\phi_{\text{com},i}$]\label{rem:phicom}
An alternative selection for $\phi_{\text{com},i}$, defined in \eqref{eq:phicomi}, is 
$\phi_{\text{com},i}'=\square(\Diamond\bbv_{j\in\ccalC_m}\wedge(\Diamond \bbv_{e\in\ccalC_h}\wedge(\Diamond \bbv_{d\in\ccalC_g}\wedge\dots)))$ that requires robot $i$ to visit communication points for all teams $\ccalT_m$, $m\in\ccalM_i$ in an given order that respects the schedules $\texttt{sched}_i$. However, using this formula, there is still no guarantee that all communication points will appear in the suffix part $\tilde{\tau}_i^{\text{suf},0}$ in an order that respects $\texttt{sched}_i$, as this depends on the structure of the LTL formula $\phi_i$ and the $\text{wTS}_i$.
Therefore, we have chosen \eqref{eq:phicomi}, instead of $\phi_{\text{com},i}'$, since \eqref{eq:phicomi} corresponds to a much smaller NBA that makes the proposed algorithm more computationally efficient.
\end{rem}

\subsection{Online Construction of Paths}\label{sec:pathni}

\begin{algorithm}[t]
\caption{Distributed construction of $\texttt{path}_i^{n_i+1}$, $\forall i\in\ccalT_m$, $\forall n_i\in\mathbb{N}$.}
\label{alg:plan}
\KwIn {Set $\ccalC_m$, $\text{wTS}_i$, $n_i$}
\KwOut{Paths: $\texttt{path}_i^{n_i+1}$, $\forall i\in\ccalT_m$}\label{plan:out}
Initialize $c_i=1$\;\label{plan:line1}
\While{$c_i\leq|\ccalM_i|$}{\label{plan:line2}
\If{team $\ccalT_m$ with $m=\ccalM_i(c_i)$ communicates}{\label{plan:line3}
\For{$j\in\ccalC_m$}{\label{plan:line4}
Define $\psi_i$ by \eqref{eq:psi} given (i) $\bbv_j$ for team $\ccalT_m$ and (ii) the selected communication points for other teams $\ccalT_h$, $h\in\ccalM_i\setminus\set{m}$\;\label{plan:line5}
Construct $P_i$ and synthesize a suffix loop $\rho_{P_i}^{\text{suf},j}$ (if it exists) around $q_{P_i}^F$ defined in \eqref{eq:rho0} that minimizes $J(\Pi|_{\text{wTS}_i}\rho_{P_i}^{\text{suf},j})$\;\label{plan:line6}
Compute $\tilde{\tau}_i^{\text{suf},j}=\Pi|_{\text{wTS}_i}\rho_{P_i}^{\text{suf},j}$\;\label{plan:line7}
}}
Define $\texttt{Cost}_{j}=\sum_{r\in\ccalT_m} J(\tilde{\tau}_r^{\text{suf},j})$, for all $j\in\ccalC_m$\;\label{plan:line8}
Compute $j^*=\argmin_{j\in\ccalC_m} \{\texttt{Cost}_{j}\}_{j\in\ccalC_m}$ \;\label{plan:line9}
Initialize paths $\texttt{path}_i^{n_i+1,c_i}=\tilde{\tau}_i^{\text{suf},j^*}$, for all $i\in\ccalT_m$\;\label{plan:line10}
\While {$\texttt{path}_i^{n_i+1,c_i}$ does not respect $\texttt{sched}_i$}{\label{plan:line11}
Update $\texttt{path}_i^{n_i+1,c_i} = \texttt{path}_i^{n_i+1,c_i}  | \tilde{\tau}_i^{\text{suf},j^*}$}\label{plan:line12}
Update $c_i=c_i+1$\;\label{plan:line13}}
Return path $\texttt{path}_i^{n_i+1}=\texttt{path}_i^{n_i+1,|\ccalM_i|}$\;\label{plan:line14}
\end{algorithm}

The construction of the paths $\texttt{path}_i^{n_i}$ occurs online and in an iterative fashion, for all $n_i\in\mathbb{N}_+$, as the robots navigate the workspace. Specifically, $\texttt{path}_i^{n_i+1}$ is constructed and updated every time robot $i$ participates at communication events, as it executes $\texttt{path}_i^{n_i}$. Hereafter, we denote by $\texttt{path}_i^{n_i+1,c_i}$ the path constructed when robot $i$ participates at the $c_i$-th communication event in $\texttt{path}_i^{n_i}$. The iteration index $c_i$ is initialized as $c_i=1$ at the beginning of execution of $\texttt{path}_i^{n_i}$ and is updated as $c_i=c_i+1$ when the path  $\texttt{path}_i^{n_i+1,c_i}$ is constructed. Once robot $i$ has participated in $|\ccalM_i|$ communication events, i.e., $c_i=|\ccalM_i|$, then the next path $\texttt{path}_i^{n_i+1}=\texttt{path}_i^{n_i+1,|\ccalM_i|}$ has been constructed and will be executed after the execution of $\texttt{path}_i^{n_i}$.

In what follows, we present the distributed construction of $\texttt{path}_i^{n_i+1}$, which is also summarized in Algorithm \ref{alg:plan} and illustrated in Figure \ref{fig:plan}.
Also, in Algorithm \ref{alg:plan}, for simplicity of notations, we assume that the indices of the teams in the sets $\ccalM_i$ are ordered as per the respective schedules $\texttt{sched}_i$. This means that if the robots in team $\ccalT_m$, $m=\ccalM_i(c_i)$, communicate then the next communication event that robot $i$ needs to participate during the execution of $\texttt{path}_i^{n_i}$ is $\ccalM_i(c_i+1)$.
Assume that the robots $i\in\ccalT_m$, $m=\ccalM_i(c_i)$, communicate during the execution of the paths $\texttt{path}_i^{n_i}$.
To design the paths $\texttt{path}_i^{n_i+1,c_i}$, the robots $i\in\ccalT_m$ need to select a new communication point $\bbv_j$, $j\in\ccalC_m$ and possibly update the waypoints $\bbv_j$, $j\in\ccalI$ so that the $\text{LTL}_{-\bigcirc}$  tasks $\phi_i$ are satisfied.
The paths $\texttt{path}_i^{n_i+1,c_i}$ are constructed in a similar way as the paths $\texttt{path}_i^{0}$ in Section \ref{sec:init}. The only difference lies in the definition of the LTL formula $\psi_i$ in \eqref{eq:psi}, since now the robots need to autonomously select a new optimal communication point for team $\ccalT_m$ given the already selected communication points for all other teams. Specifically, all robots $i\in\ccalT_m$ perform in parallel the following two steps for all candidate new communication points $\bbv_j$, $j\in\ccalC_m$, for team $\ccalT_m$ [lines \ref{plan:line2}-\ref{plan:line4}, Alg. \ref{alg:plan}]. First, every robot $i\in\ccalT_m$ constructs the LTL formula $\psi_i$, defined in \eqref{eq:psi}, for every candidate new communication point $\bbv_j$, $j\in\ccalC_m$ for team $\ccalT_m$, and given the already selected communication points for all other teams $\ccalT_h$, $h\in\ccalM_i\setminus\set{m}$; see \eqref{eq:psi} [line \ref{plan:line5}, Alg. \ref{alg:plan}]. Second, given the $\text{wTS}_i$ and the NBA $B_i$ that corresponds to $\psi_i$, every robot $i\in\ccalT_m$ constructs the corresponding PBA $P_i=\text{wTS}_i\otimes B_i$ and computes the \textit{optimal} suffix loop, denoted by $\rho_{P_i}^{\text{suf},j}$, around the same PBA final state $q_{P_i}^F=(q_{\text{wTS}_i}^F,q_B^F)$ that was used to construct the initial suffix loop of $\rho_{P_i}^0 $ in \eqref{eq:rho0}. Note that by optimal suffix loop $\rho_{P_i}^{\text{suf},j}$, we refer to the path that minimizes the cost $J(\Pi|_{\text{wTS}_i}\rho_{P_i}^{\text{suf},j})$.
The projection of this optimal suffix loop $\rho_{P_i}^{\text{suf},j}$ on the state-space of $\text{wTS}_i$ is denoted by $\tilde{\tau}_i^{\text{suf},j}$ [lines \ref{plan:line6}-\ref{plan:line7}, Alg. \ref{alg:plan}].

\begin{figure}[t]
    \centering
    \subfigure[Communication within $\ccalT_1$]{
    \label{com1}
    \includegraphics[width=0.44\linewidth]{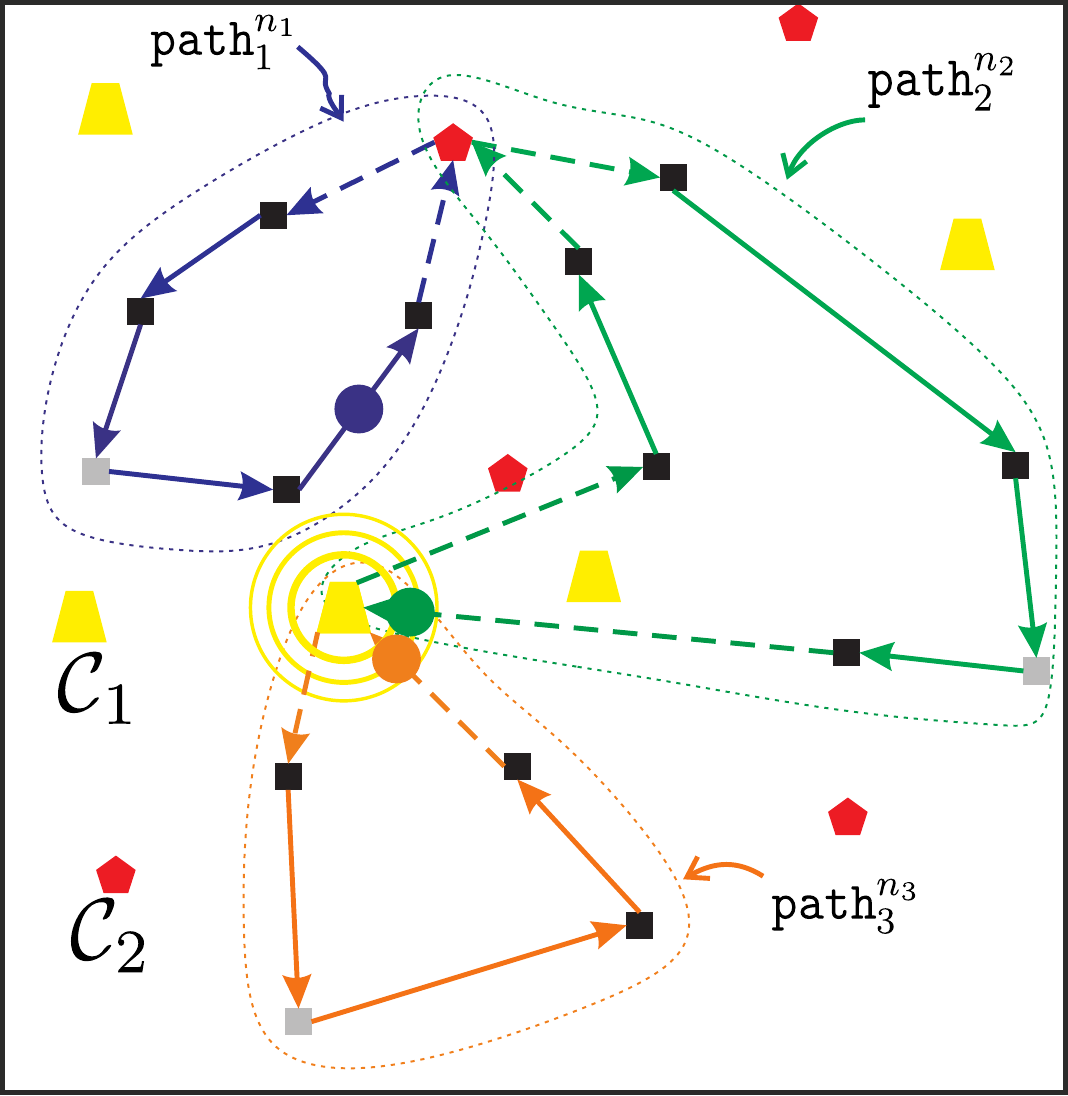}} 
    \subfigure[Selection of new $\bbv_j$, $j\in\ccalC_1$]{
    \label{upd1}
       \includegraphics[width=0.44\linewidth]{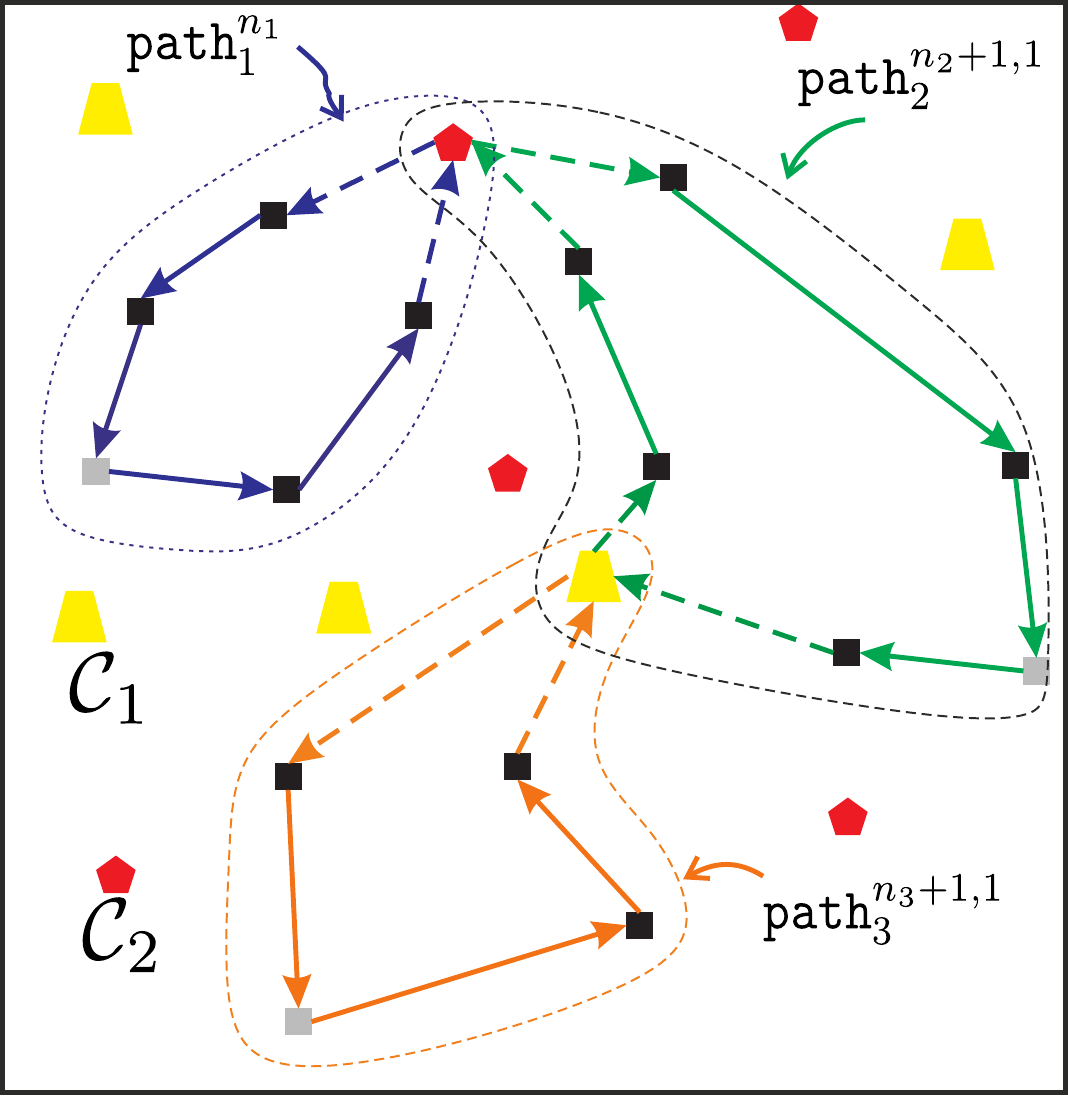}}                                
           \caption{Illustration of Algorithm \ref{alg:plan} for network of $N=3$ robots (colored dots) with schedules $\texttt{sched}_1=[X,2]^\omega$, $\texttt{sched}_2=[1,2]^\omega$, and $\texttt{sched}_3=[1,X]^\omega$. All robots currently execute paths $\texttt{path}_i^{n_i}$ constructed by Algorithm \ref{alg:plan}. Figure \ref{com1} illustrates the communication events within team $\ccalT_1$. The corresponding paths  $\texttt{path}_i^{n_i+1,c_i}$ constructed at this communication event is depicted in Figure \ref{upd1}. Observe in Figure \ref{upd1} that robots $3$ has finalized the construction of the paths $\texttt{path}_3^{n_3+1}$ since $|\ccalM_3|=1$. The gray square denotes the state $\Pi|_{\text{wTS}_i}q_{P_i}^F$. 
        }                            
     \label{fig:plan}
\end{figure}

Once all robots $i\in\ccalT_m$ have constructed the suffix parts $\tilde{\tau}_i^{\text{suf},j}$ for all $j\in\ccalC_m$, they compute the total cost  $\texttt{Cost}_{j}=\sum_{i\in\ccalT_m} J(\tilde{\tau}_i^{\text{suf},j})$ [line \ref{plan:line8}, Alg. \ref{alg:plan}]. This cost captures the  distance that all robots $i\in\ccalT_m$ need to travel during a single execution of the suffix parts $\tilde{\tau}_i^{\text{suf},j}$ if the new communication point for team $\ccalT_m$ is $\bbv_j$, $j\in\ccalC_m$. Among all the suffix parts $\tilde{\tau}_i^{\text{suf},j}$, all robots $i\in\ccalT_m$ select the suffix part $\tilde{\tau}_i^{\text{suf},j^*}$, with  $j^*=\argmin_{j} \{\texttt{Cost}_{j}\}_{j\in\ccalC_m}$ [line \ref{plan:line9}, Alg. \ref{alg:plan}]. 

Given the optimal suffix part $\tilde{\tau}_i^{\text{suf},j^*}$, we construct $\texttt{path}_i^{n_i+1,c_i}$ exactly as the initial paths $\texttt{path}_i^0$. Specifically, first, the paths $\texttt{path}_i^{n_i+1,c_i}$ are initialized as 
 $\texttt{path}_i^{n_i+1,c_i}=\tilde{\tau}_i^{\text{suf},j^*}$ [line \ref{plan:line10}, Alg. \ref{alg:plan}]. Then, we append  $\tilde{\tau}_i^{\text{suf},j^*}$ to $\texttt{path}_i^{n_i+1,c_i}$ as many times as needed to satisfy the schedules $\texttt{sched}_i$ [lines \ref{plan:line11}-\ref{plan:line12}, Alg. \ref{alg:plan}]. Note that since the state $q_i^{\bbv_j}$, $j\in\ccalC_m$ appears at least once in the suffix part of $\tilde{\tau}_i^{\text{suf},j^*}$, for all $m\in\ccalM_i$, the suffix part $\tilde{\tau}_i^{\text{suf},j^*}$ will be appended at most $|\ccalM_i|-1$ times to $\texttt{path}_i^{n_i+1,c_i}$. After the construction of $\texttt{path}_i^{n_i+1,c_i}$, the iteration index $c_i$ is updated as $c_i=c_i+1$ and points to the next path $\texttt{path}_i^{n_i+1,c_i}$ that will be constructed when robot $i$ communicates with the robots in team $\ccalT_h$, $h=\ccalM_i(c_i)$ [line \ref{plan:line12}, Alg. \ref{alg:plan}]. \footnote{Note that the next communication event $\ccalM_i(c_i)$ respects the schedules $\texttt{sched}_i$, by construction of $\ccalM_i$. }
If $c_i=|\ccalM_i|$, then this corresponds to the last communication event that robot $i$ needs to participate during the execution of $\texttt{path}_i^{n_i}$ and, therefore, the construction of $\texttt{path}_i^{n_i+1}$ is finalized, i.e., $\texttt{path}_i^{n_i+1}=\texttt{path}_i^{n_i+1,|\ccalM_i|}$ [line \ref{plan:line14}, Alg. \ref{alg:plan}]. In this case, $c_i$ is re-initialized as $c_i=1$ [line \ref{plan:line1}, Alg. \ref{alg:plan}]. 


\begin{rem}[Implicit synchronization across robots]\label{rem:ni}
While the robots transition from $\texttt{path}_i^{n_i}$ to $\texttt{path}_i^{n_i+1}$ asynchronously, there is an implicit synchronization in the system since, for any iteration $n\in\mathbb{N}_+$, the robots that finish the execution of $\texttt{path}_i^{n}$, first cannot finish the execution of $\texttt{path}_i^{n+1}$ until all other robots $r$ have finished the execution of their paths $\texttt{path}_r^{n}$. The reason is that (i) every robot $i$ has to participate in $|\ccalM_i|$ communication events during the execution of $\texttt{path}_i^{n}$ and (ii) the graph of teams $\ccalG_\ccalT$ is connected by construction of the teams. Therefore, if there exist robots $i$ and $r$ where robot $i$ executes the path $\texttt{path}_i^{n+2}$ and robot $r$ executes the path $\texttt{path}_r^{n}$ it must be the case that robot $i$ has skipped at least one communication event during the execution of $\texttt{path}_i^{n+1}$, which cannot happen by construction of the proposed algorithm. Therefore, there exist time instants $t_n$ so that  $\texttt{path}_i^{n_i}=\texttt{path}_i^{n}$, for every  $n\in\mathbb{N}_+$ and for all $i\in\ccalN$. 
\end{rem}

\begin{rem}[Computational Cost]\label{rem:comp}
Note that to design the path $\texttt{path}_i^{n_i+1,c_i}$, every robot $i$ needs to solve $|\ccalC_m|$ optimal control synthesis problems. Therefore, the computational cost of Algorithm \ref{alg:plan} increases with $|\ccalC_m|$. To reduce the computational burden, Algorithm \ref{alg:plan} can be executed over subsets $\bar{\ccalC}_m\subseteq\ccalC_m$ that can change with iterations $n_i$ but always include the current communication point for team $\ccalT_m$. The latter is required to ensure that paths $\texttt{path}_i^{n_i}$ can be synthesized for all $n_i>0$, if a solution to Problem \ref{pr:pr1} exists; see Proposition \ref{prop:feas}. Moreover, sampling-based approaches can be used to synthesize the suffix parts $\tilde{\tau}_i^{\text{suf},j}$ that do not require the explicit construction of the PBA or the application of computationally expensive graph-search methods \cite{kantaros2017Csampling}. Finally, in Proposition \ref{prop:converg2}, we show that Algorithm \ref{alg:plan} terminates after a finite number of iterations, i.e., a repetitive pattern in the paths $\texttt{path}_i^{n_i}$ is eventually detected, for all $i\in\ccalN$. This means that the computational cost is bounded.
\end{rem}  

\begin{rem}[Fixed final state $q_{P_i}^F$]
Recall that the fixed PBA final state $q_{P_i}^F$, defined in \eqref{eq:rho0}, is used to construct the paths $\texttt{path}_i^{n_i+1}$, for all $n_i\in\mathbb{N}$ and for all $i\in\ccalN$,  This requirement can be relaxed by defining  the paths $\texttt{path}_i^{n_i+1,c_i}$ as $\texttt{path}_i^{n_i+1,c_i}=\Pi_{\text{wTS}_i}\rho_{c_i}$, where $\rho_{c_i}=\rho_{c_i,1}|\rho_{c_i,2}|\dots,|\rho_{c_i,K}$ is a feasible path in the state-space of $P_i$, $\rho_{c_i,k}$ a feasible path in the state-space of $P_i$ that connects two possibly different PBA final states, for all $k\in\set{1,\dots,K}$, and $K<|\ccalM_i|$ is determined so that  execution of $\texttt{path}_i^{n_i+1,c_i}$, for any $c_i$, ensures that robot $i$ will communicate exactly once with all teams $\ccalT_m$, $m\in\ccalM_i$.\footnote{Observe that if all paths $\rho_{c_i,k}$ are defined as the shortest loops around $q_{P_i}^F$, then $\rho_{c_i,k}$ coincides with the $\rho_{P_i}^{\text{suf},j}$, for all $k\in\set{1,\dots,K}$.}
In this case $\texttt{path}_i^{n_i+1}$ is not a periodic path that can be executed infinitely and, therefore, \eqref{eq:taui} cannot be used to model the solution of Algorithm 2, which will now be an infinite aperiodic sequence of states. 
Also, allowing the paths $\texttt{path}_i^{n_i+1}$ to be associated with multiple PBA final states would increase the computational burden of Algorithm \ref{alg:plan}, as it requires the computation of $K$ paths in the PBA $P_i$.
%
\end{rem}




\subsection{Asynchronous Execution}\label{sec:execution}
%
In the majority of global LTL-based motion planning, robots are assumed to execute their assigned motion plans synchronously, i.e., all the robots pick synchronously their next states, see e.g.,\cite{kloetzer2010automatic,kantaros15asilomar}. However, assuming that robot motion is performed in a synchronous way is conservative due to, e.g., uncertainty and exogenous disturbances in the arrival times of the robots at their next locations as per the discrete path $\texttt{path}_i^{n_i}$. To the contrary, here the discrete plans $\texttt{path}_i^{n_i}$ are executed asynchronously across the robots, as per Algorithm \ref{alg:implement}. 

In Algorithm \ref{alg:implement}, $\texttt{path}_i^{n_i}(\kappa_i)$ stands for the $\kappa_i$-th state of the discrete path $\texttt{path}_i^{n_i}$. The different indices $\kappa_i$ for the robots's states in the plans $\texttt{path}_i^{n_i}$ allow us to model the situation where the robots pick asynchronously their next states in $\text{wTS}_i$. Also, in Algorithm \ref{alg:implement}, the set $\ccalK_i^{n_i}$ collects an index $\kappa_i^m$ for all teams $\ccalT_m$, $m\in\ccalM_i$ that (i) satisfy $\texttt{path}_i^{n_i}(\kappa_i^m)=q_i^{\bbv_j}$, where $q_i^{\bbv_j}$ is associated with a communication point $\bbv_j$, $j\in\ccalC_m$, $m\in\ccalM_i$ and (ii) respect the schedules as described in Section \ref{sec:init}. Note that such indices  $\kappa_i^m$ exist by construction of the paths $\texttt{path}_i^{n_i}$.
According to Algorithm \ref{alg:implement}, when the state of robot $i$ is $\texttt{path}_i^{n_i}(\kappa_i)=q_i^{\bbv_j}$, $j\in\ccalI$ i.e., when robot $i$ arrives at a location $\bbv_j$ in the workspace, it checks if $\kappa_i\in\ccalK_i^{n_i}$ [lines \ref{implement:line2}-\ref{implement:line3}, Alg. \ref{alg:implement}]. If so, then robot $i$ performs the following control policy [line \ref{implement:line4}, Alg. \ref{alg:implement}]:
\begin{defn}[Control policy at communication locations]\label{def:policy}
Every robot $i$ that arrives at a communication location $\bbv_j$, $j\in\ccalC_m$, $m\in\ccalM_i$, selected by Algorithm \ref{alg:plan} waits there indefinitely, or until all other robots in the team arrive.
\end{defn}

When all the other robots of team $\ccalT_m$ arrive at the communication location $\bbv_j$, $j\in\ccalC_m$, communication for team $\ccalT_m$ occurs and Algorithm \ref{alg:plan} is executed to synthesize $\texttt{path}_i^{n_i+1,c_i}$ [lines \ref{implement:line5}-\ref{implement:line6}, Alg. \ref{alg:implement}]. After that, robot $i$ moves towards the next state $\texttt{path}_i^{n_i}(\kappa_i+1)$ [line \ref{implement:for}, Alg. \ref{alg:implement}]. In line \ref{implement:for} of Alg. \ref{alg:implement}, $K_i^{n_i}$ denotes the number of waypoints/states in $\texttt{path}_i^{n_i}$. This process is repeated until robot $i$ visits all locations in $\texttt{path}_i^{n_i}$.  Once robot $i$ visit all waypoints of $\texttt{path}_i^{n_i}$, it starts executing the path $\texttt{path}_i^{n_i+1}$ [line \ref{implement:next}, Alg. \ref{alg:implement}]. If $n_i$ is the last iteration of Algorithm \ref{alg:plan}, then $\texttt{path}_i^{n_i}$ is executed indefinitely. 
\begin{algorithm}[t]
\caption{Asynchronous execution of $\texttt{path}_i^{n_i}$}
\label{alg:implement}
\KwIn {Discrete path $\texttt{path}_i^{0}$  and set $\ccalK_i^0$}
$n_i=0$\;
 \For{$\kappa_i=1:K_i^{n_i}$}{\label{implement:for}
	Move towards the state $\texttt{path}_i^{n_i}(\kappa_i)$\;\label{implement:line2}
	\If {$\kappa_i\in\ccalK_i^{n_i}$}{ \label{implement:line3}
		\textit{Wait} at communication point $\bbv_j$, $j\in\ccalC_m$ [Definition \ref{def:policy}]\;  \label{implement:line4}
		\If{all robots in $\ccalT_m$ are present at node $\bbv_j$}{ \label{implement:line5}
			Communication occurs within team~$\ccalT_m$ and execution of Algorithm \ref{alg:plan} \; \label{implement:line6}
			}}}
			Execute the next path $\texttt{path}_i^{n_i+1}$\;\label{implement:next}
\end{algorithm}

\section{Algorithm Analysis}\label{sec:analysis}

In this section, we  present results pertaining to completeness and optimality of the proposed distributed control framework. Specifically, in Section \ref{sec:compl}, we show that if there exists a solution to Problem \ref{pr:pr1}, then the proposed distributed framework will generate prefix-suffix plans $\tau_i^{n_i}$, defined in \eqref{eq:taui}, that can be executed asynchronously according to Algorithm \ref{alg:implement}, and satisfy the assigned LTL tasks and the intermittent connectivity requirement, for every iteration $n_i\geq 0$. Then, in Section \ref{sec:opt} we show that the cost of the suffix part of the plans in \eqref{eq:taui} decreases with every iteration of Algorithm 2 while in Section \ref{sec:complex} we show that these plans converge in a finite number of iterations. Note that since the proposed algorithm is online, synthesis and execution take place concurrently and this is reflected in the subsequent results.
\subsection{Completeness}\label{sec:compl}

First, we show that if there exists a feasible solution to Problem \ref{pr:pr1} then, feasible paths $\texttt{path}_i^{n_i}$ i.e., feasible loops $\rho_{P_i}^{n_i}$ defined over the state-space of the corresponding PBA $P_i$, can be designed, for all $n_i\in\mathbb{N}$. This implies that Algorithm \ref{alg:plan} can generate plans $\tau_i^{n_i}$, for any $n_i\geq 0$ and that robots $i$ in any team $\ccalT_m$, for $m\in\ccalM_i$, can stop executing Algorithm \ref{alg:plan} at any iteration $n_i^m\geq 0$.
%

\begin{prop}[Feasibility]\label{prop:feas}
Assume that there exists a solution to Problem \ref{pr:pr1}. Then, feasible plans $\texttt{path}_i^{n_i}$ can be constructed for all $n_i\geq 0$.
\end{prop}

\begin{proof}
First observe that if there exists a solution to Problem \ref{pr:pr1}, then feasible initial paths $\tilde{\tau}_i^0$ that satisfy $\psi_i$ in \eqref{eq:psi}, for all robots $i\in\ccalN$, will be detected since at initialization we exhaustively search through all available communication points assigned to the teams $\ccalT_m$, $m\in\ccalM$, as shown in Lemma \ref{lem:init}. Therefore, initial feasible paths $\texttt{path}_i^{0}$ can be constructed. Then, to prove this result, it suffices to show that if there exists a feasible path  $\texttt{path}_i^{n_i}$, then Algorithm \ref{alg:plan} can construct a feasible path  $\texttt{path}_i^{n_i+1}$ for all $n_i\geq 0$.  This means that Algorithm \ref{alg:plan} will not deadlock. Note that Algorithm \ref{alg:plan} does not search over all combinations of communication points assigned to the teams.

In what follows, we show by induction that if there exists a feasible path $\texttt{path}_i^{n_i}$ then, Algorithm \ref{alg:plan} will construct feasible paths  $\texttt{path}_i^{n_i+1,c_i}$ for all $c_i\in\set{1,\dots,|\ccalM_i|}$ and, consequently, it will construct a feasible path $\texttt{path}_i^{n_i+1,|\ccalM_i|}=\texttt{path}_i^{n_i+1}$ for all $n_i\geq 0$. To show this, we first define the sets $\ccalF_{c_i}^{n_i+1}$ that collect the suffix parts $\tilde{\tau}_i^{\text{suf},j}$ constructed by Algorithm \ref{alg:plan} during the construction of $\texttt{path}_i^{n_i+1,c_i}$, for all $c_i\in\set{1,\dots,|\ccalM_i|}$. Now, assume that there exists a feasible path $\texttt{path}_i^{n_i}$. This means that $\ccalF_{0}^{n_i+1}:=\{ \tilde{\tau}_i^{\text{suf},j^*,n_i}\}\neq\emptyset$, where $\tilde{\tau}_i^{\text{suf},j^*,n_i}$ is the suffix part used for the construction of the path $\texttt{path}_i^{n_i}$. 
%
First, we show that $\ccalF_{1}^{n_i+1}\neq\emptyset$, i.e., that Algorithm \ref{alg:plan} will construct a feasible plan $\texttt{path}_i^{n_i+1,1}$. Note that the only difference between the paths $\texttt{path}_i^{n_i+1,1}$ and $\texttt{path}_i^{n_i}=\texttt{path}_i^{n_i,|\ccalM_i|}$, in terms of the selected communication points for teams $\ccalT_m$, $m\in\ccalM_i$, lies in the selected communication point of exactly one team $\ccalT_m$, $m\in\ccalM_i$. Also, recall that Algorithm \ref{alg:plan} searches over all communication points $j\in\ccalC_m$, including the current communication point of $\ccalT_m$ that appears in $\texttt{path}_i^{n_i}$, to select the new communication point for team $\ccalT_m$. Therefore, there exists an optimal control synthesis problem that is solved by Algorithm \ref{alg:plan} during the computation of $\texttt{path}_i^{n_i+1,1}$ such that the
LTL formula $\psi_i$ is defined over the communication points selected in $\texttt{path}_i^{n_i,|\ccalM_i|}$.
Since this optimal control synthesis problem is feasible, by the assumption that $\texttt{path}_i^{n_i}$ is a feasible path, the generated suffix part, which was also used to construct $\texttt{path}_i^{n_i,|\ccalM_i|}$, belongs to $\ccalF_{1}^{n_i+1}$, i.e., $\ccalF_{1}^{n_i+1}\neq\emptyset$. 
The inductive step follows. Assume that $\ccalF_{c_i}^{n_i+1}\neq\emptyset$. Then, following the same logic as before we can show that the feasible suffix path used to construct  $\texttt{path}_i^{n_i+1,c_i}$ belongs to $\ccalF_{c_i+1}^{n_i+1}$, i.e., $\ccalF_{c_i+1}^{n_i+1}\neq\emptyset$. By induction we conclude that if $\ccalF_{0}^{n_i+1}\neq\emptyset$, i.e., if there exists a feasible path  $\texttt{path}_i^{n_i}$, then $\ccalF_{c_i}^{n_i+1}\neq\emptyset$ for all $c_i\in\set{1,\dots,|\ccalM_i|}$ and all $n_i\geq 0$ completing the proof. \end{proof}
 


To prove task satisfaction and intermittent communication, we also need to show that the network is \textit{deadlock-free} when the paths $\texttt{path}_i^{n_i}$ are executed according to Algorithm \ref{alg:implement}. Specifically, we assume that there is a \textit{deadlock}, if there are robots of any team $\ccalT_m$ that are waiting forever at a communication point, selected by Algorithm \ref{alg:plan}, for the arrival of all other robots of team $\ccalT_m$ due to the control policy in Definition \ref{def:policy}. 

\begin{prop}[Deadlock-free]\label{prop:deadlock}
The mobile robot network is deadlock-free when the paths $\tau_i^{n_i}$ in \eqref{eq:taui} are executed according  to Algorithm \ref{alg:implement}.
\end{prop}
\begin{proof}
Let $\ccalW_{\bbv_e}\subset \ccalT_m$ denote the set of robots that are waiting at communication point $\bbv_e$, $e\in\ccalC_m$, selected by Algorithm \ref{alg:plan}, for the arrival of the other robots that belong to team $\ccalT_m$. Assume that the robots in $\ccalT_m\backslash \ccalW_{\bbv_e}$ never arrive at that node so that communication at node $\bbv_e$ for team $\ccalT_m$ never occurs. This means that the robots in $\ccalT_m\backslash \ccalW_{\bbv_e}$ are waiting indefinitely at communication locations $\bbv_j\in\ccalC_n$, $j\neq e$, $n\neq m$, $n\in\ccalN_{\ccalT_m}$, selected by Algorithm \ref{alg:plan}, to communicate with robots in team $\ccalT_n$. The fact that there are robots that remain indefinitely at node $\bbv_j\in\ccalC_n$ means that a communication within team $\ccalT_n$ never occurs by construction of Algorithm \ref{alg:implement}. Following an argument similar to the above, we conclude that the robots in $\ccalT_n\backslash \ccalW_{\bbv_j}$ are waiting indefinitely at nodes $\bbv_{k\neq j}\in\ccalC_f$ to communicate with robots that belong to a team $\ccalT_f$, $f\in\ccalN_{\ccalT_n}$. Therefore, if a communication event never occurs for team $\ccalT_m$, then all robots $i\in\ccalN$ need to be waiting at communication locations selected by Algorithm \ref{alg:plan} and, consequently, there is no communication location where all robots are present, i.e., there is no team within which communication will ever occur. Throughout the rest of the proof we will refer to this network configuration as a \textit{stationary configuration}.
      
In what follows, we show by contradiction that the network can never reach a stationary configuration when the paths in \eqref{eq:taui} are executed asynchronously as per Algorithm \ref{alg:implement}. To show this result, we 
we first model the asynchronous execution of the schedules $\texttt{sched}_i$, constructed by Algorithm \ref{alg:schedule}, as per Algorithm \ref{alg:implement}. 
Specifically, we introduce discrete time steps $z_i$ that are initialized as $z_i=1$ and are updated as $z_i=z_i+1$ asynchronously across the robots as follows. If at the current discrete time step $z_i$ robot $i$ participates in the communication event $\texttt{sched}_i(z_i)=m$, for some $z_i\in\mathbb{N}_+$ and $m\in\ccalM_i$, then robot $i\in\ccalT_m$ waits until all the other robots in team $\ccalT_m$ are available to communicate. Once all robots in $\ccalT_m$ are available, the discrete time step $z_i$ is updated as $z_i=z_i+1$. If $\texttt{sched}_i(z_i)=X$, then robot $i$ updates $z_i=z_i+1$ without waiting. 

Using this model to describe asynchronous execution of the schedules, we now show by contradiction that if the network gets trapped at a stationary configuration, then there exist robots of some team $\ccalT_m$ that missed a communication event at node $\bbv_e$, $e\in\ccalC_m$, at a previous time instant, which cannot happen by construction of Algorithm \ref{alg:implement}. Consider that there is an arbitrary time instant $t_0$ at which the network is at a stationary configuration and let the current communication event for all robots $i\in\ccalT_m$ be  $\texttt{sched}_i(n_i^{\ccalT_m}(t_0))=m$ for some $m\in\ccalM_i$, where the indices $n_i^{\ccalT_m}$ were defined in Algorithm \ref{alg:schedule}. Define also the set $\ccalN_{\text{min}}(t_0)=\left\{n_i^{\ccalT_m}(t_0)|n_i^{\ccalT_m}(t_0)=\min\{n_e^{\ccalT_g}(t_0)\right\}_{e=1}^N, g\in\ccalM_e\}$ that collects the smallest indices $n_i^{\ccalT_m}(t_0)$ among all robots. Also let $n_e^{\ccalT_g}(t_0)$ be an index such that $n_e^{\ccalT_g}(t_0)\in\ccalN_{\text{min}}(t_0)$. By assumption there are robots $e\in\ccalT_g$ and $r\in\ccalT_z$, $g\in\ccalN_{\ccalT_z}$, such that $e\in\ccalW_{\bbv_f}(t_0)$, $\bbv_f\in\ccalT_g$  and $r\in\ccalW_{\bbv_d}(t_0)$, $\bbv_d\in\ccalT_z$, and, therefore, the events that are taking place for these two robots according to their assigned schedules of meeting events are $\texttt{sched}_e(n_e^{\ccalT_g}(t_0))=g$ and $\texttt{sched}_r(n_r^{\ccalT_z}(t_0))=z$. Since $n_e^{\ccalT_g}(t_0)\in\ccalN_{\text{min}}(t_0)$ we have that $n_e^{\ccalT_g}(t_0)\geq n_r^{\ccalT_z}(t_0)$, which along with the fact that $g\in\ccalN_{\ccalT_z}$ results in $n_e^{\ccalT_g}(t_0)> n_r^{\ccalT_z}(t_0)$ by construction of Algorithm \ref{alg:schedule}. This leads to the following contradiction. The fact that $n_e^{\ccalT_g}(t_0)>n_r^{\ccalT_z}(t_0)$ means that there exists a time instant $t<t_0$ at which the event that took place for robots $a\in\ccalT_g\cap\ccalT_z$ was $\texttt{sched}_a (n_r^{\ccalT_g}(t))=g$ and at least one of these robots did not wait for the arrival of all other robots in team $\ccalT_g$, since at the current time instant $t_0$ there are still robots in team $\ccalT_g$ waiting for the arrival of other robots. However, such a scenario is precluded by construction of Algorithm \ref{alg:implement}. Consequently, the asynchronous execution of the
schedules $\texttt{sched}_i$ as per Algorithm \ref{alg:implement} is deadlock-free. Recall now that the paths \eqref{eq:taui} respect the schedules $\texttt{sched}_i$ and  that it is not possible that there exist robots in any team $\ccalT_m$ that wait for other robots in the same team at different communication points $\bbv_j$, $j \in \ccalC_m$. Thus, we conclude that the network is deadlock-free when the plans \eqref{eq:taui} are executed asynchronously, as per Algorithm \ref{alg:implement}, which completes the proof.
\end{proof}
\begin{rem}[Bounded waiting times]
Proposition \ref{prop:deadlock} shows also that the waiting times introduced by Algorithm \ref{alg:implement} are~bounded.
\end{rem}

In Theorems \ref{thm:task}-\ref{thm:com}, we show that the assigned local tasks $\phi_i$ and the intermittent connectivity requirement captured by \eqref{eq:globalLTL} are satisfied.

\begin{thm}[Task satisfaction]\label{thm:task}
The asynchronous execution of the motion plans $\tau_i^{n_i}$ in \eqref{eq:taui} as per Algorithm \ref{alg:implement}, satisfies the $\text{LTL}_{-\bigcirc}$ statements $\phi_i$, i.e., $\tau_i^{n_i}\models\phi_i$, for any $n_i\geq 0$ and all robots $i\in\ccalN$.
\end{thm}

\begin{proof}
First observe that Algorithm \ref{alg:plan} can design feasible paths $\texttt{path}_i^{n_i}$, for any $n_i\geq 0$ as long as there exists a solution to Problem \ref{pr:pr1}, due to Proposition \ref{prop:feas}. 
Moreover, the waiting times at the communication points in the plans $\tau_i^{n_i}$ are bounded by Proposition \ref{prop:deadlock}. Therefore, the infinite paths $\tau_i^{n_i}$ will be executed without any deadlocks. This is necessary to
satisfy $\phi_i$, as LTL formulas are satisfied by infinite sequences of states in $\text{wTS}_i$.

To prove this result, first we need to show that all transitions in $\text{wTS}_i$ that are generated by the plans  in \eqref{eq:taui} respect the transition rule $\rightarrow_i$; see Definition \ref{defn:wTS}. Next, we need to show that the infinite run $\rho_{B_i}$ of the NBA $B_i$ that corresponds to $\phi_i$ over the words $\sigma_i^{n_i}$ generated during the execution of $\tau_i^{n_i}$ is accepting, i.e., \footnote{The generated word $\sigma_i^{n_i}$, called also trace of ${\tau}_i$ \cite{clarke1999model} and denoted by $\texttt{trace}({\tau}_i)$, is defined as $\sigma_i^{n_i}=\texttt{trace}({\tau}_i^{n_i}):=L_{i}({\tau}_{i}^{n_i}(1))L_{i}({\tau}_{i}^{n_i}(2))\dots$, where $L_i$ is the labeling function defined in Definition \ref{defn:wTS}.}
\begin{equation}\label{eq:show}
\texttt{Inf}(\rho_{B_i})\cap\ccalF_{B_i}\neq\emptyset.
\end{equation}

First, we show that all transitions in $\text{wTS}_i$ that are due to the plans in \eqref{eq:taui} respect the transition rule $\rightarrow_i$. Notice that all transitions incurred by the finite path $\texttt{path}_i^{n_i}$ respect the transition rule $\rightarrow_i$, for all $n_i\in\mathbb{N}$, by construction; see Algorithm \ref{alg:plan}. Next, we show that the transition from the last state in $\texttt{path}_i^{n_i}$ to the first state in $\texttt{path}_i^{n_i+1}$ also respects the transition rule $\rightarrow_i$, for all $n_i\in\mathbb{N}$. To show this, observe that the last state in $\texttt{path}_i^{n_i}$ is the last state in the suffix part $\tilde{\tau}_i^{\text{suf},j^*}$ used to construct $\texttt{path}_i^{n_i}$, for all $n_i\in\mathbb{N}$. Also, notice that the first state in $\texttt{path}_i^{n_i+1}$ is the state $\Pi|_{\text{wTS}_i}q_{P_i}^F$, for all $n_i\in\mathbb{N}$, which is also the first state in $\tilde{\tau}_i^{\text{suf},j^*}$. Therefore, by construction of $\tilde{\tau}_i^{\text{suf},j^*}$, the transition from the last state in $\texttt{path}_i^{n_i}$ to the first state in $\texttt{path}_i^{n_i+1}$  respects $\rightarrow_i$, for all $n_i\in\mathbb{N}$. Consequently, the plans in \eqref{eq:taui} respect $\rightarrow_i$.

Next, we show that \eqref{eq:show} holds for the plans $\tau_i^{n_i}$ in \eqref{eq:taui}, for all $n_i\geq 1$. The same logic also applies to the plans  $\tau_i^0$ in \eqref{eq:taui0}.
To show this result, recall that the paths $\texttt{path}_i^{n_i}$, for all $n_i\geq 1$ are designed by (i) constructing a suffix path $\rho_{P_i}^{\text{suf},j^*}$ that lives in the state-space $\ccalQ_{P_i}$ around the fixed PBA final state $q_{P_i}^F$ defined in \eqref{eq:rho0}, and initializing $\texttt{path}_i^{n_i}=\Pi|_{\text{wTS}_i}\rho_{P_i}^{\text{suf},j^*}$,
(ii) appending the path $\Pi|_{\text{wTS}_i}\rho_{P_i}^{\text{suf},j^*}$ as many times as needed so that $\texttt{path}_i^{n_i}$ respects the schedule $\texttt{sched}_i$. Thus, $\texttt{path}_i^{n_i}$ can be written as the projection onto $\text{wTS}_i$ of the finite path $p_i^{n_i}=\rho_{P_i}^{\text{suf},j^*}|\rho_{P_i}^{\text{suf},j^*}|\dots|\rho_{P_i}^{\text{suf},j^*}$, which means that $p_i^{n_i}$ visits the fixed PBA final state $q_{P_i}^F$ a finite number of times. Consequently, since the plans in \eqref{eq:taui} are defined as infinite sequences of paths $\texttt{path}_i^{n_i}$, we get that $q_{P_i}^F$ is visited infinitely often, i.e., \eqref{eq:show} holds, completing the proof.
\end{proof}

\begin{thm}[Intermittent Communication]\label{thm:com}
The asynchronous execution of the motion plans $\tau_i^{n_i}$ in \eqref{eq:taui} as per Algorithm \ref{alg:implement}, satisfies the intermittent communication requirement captured by the global LTL statement $\phi_{\text{com}}$, for~all~$n_i\geq 0$.
\end{thm}

\begin{proof}
By construction of the paths $\texttt{path}_i^{n_i}$  every robot $i$ will communicate once with all teams $\ccalT_m$,  $m\in\ccalM_i$,
during a single execution of the path $\texttt{path}_i^{n_i}$. Moreover, by Proposition \ref{prop:deadlock}, there are no deadlocks during the execution of the plans $\tau_i^{n_i}$.
 %
Consequently, all robots $i$ communicate infinitely often with all teams $\ccalT_m$, $m\in\ccalM_i$ completing the proof.\end{proof}

Combining the previous results, we can show that the proposed control scheme is complete.
\begin{thm}[Completeness]\label{thm:complete}
If there exists a solution to Problem \ref{pr:pr1}, Algorithm \ref{alg:plan} will find motion plans $\tau_i^{n_i}$ as in \eqref{eq:taui} that, when executed asynchronously as per Algorithm \ref{alg:implement}, satisfy the local $\text{LTL}_{-\bigcirc}$ tasks $\phi_i$ and the global LTL intermittent connectivity requirement $\phi_{\text{com}}$.
\end{thm}
\begin{proof}
By Proposition \ref{prop:feas}, we get that if there exists a solution to Problem \ref{pr:pr1}, then prefix-suffix motion plans as in  \eqref{eq:taui} will be generated for any $n_i\geq 0$. Due to Theorems \ref{thm:task} and \ref{thm:com}, the asynchronous execution of these plans as per Algorithm \ref{alg:implement} satisfies the local $\text{LTL}_{-\bigcirc}$ tasks $\phi_i$ and the intermittent communication requirement captured by the global LTL statement $\phi_{\text{com}}$. This completes the proof.\end{proof}



\subsection{Optimality} \label{sec:opt}


As discussed in Remark \ref{rem:ni}, execution of the plans in \eqref{eq:taui} is synchronized implicitly so that there exists a time instant $t_n$ when all robots execute the path $\texttt{path}_i^{n}$. In the following proposition, we examine the optimality of the paths $\texttt{path}_i^{n}$ in terms of the total cost $ \sum_{i\in\ccalN}J(\texttt{path}_i^{n})$, for any $n\in\mathbb{N}$. 

\begin{prop}[Optimality]\label{prop:distance}
Algorithm \ref{alg:plan} generates discrete paths $\texttt{path}_i^{n+1}$ so that
\begin{equation}\label{dist1}
 \sum_{i\in\ccalN}J(\texttt{path}_i^{n})\leq \sum_{i\in\ccalN}J(\texttt{path}_i^{n+1}),
\end{equation}
for all $n\geq 0$.
\end{prop}

\begin{proof}
%
Consider the discrete paths $\texttt{path}_i^n$, for some fixed $n\geq 0$. Recall that the robots may start executing the paths $\texttt{path}_i^n$ asynchronously, i.e., at different time instants. Therefore, given a time instant $t$, we divide the robots $i\in\ccalN$ in the following five disjoint sets. 
First, we collect in the set $\mathcal{R}^{n-1}(t)$ the robots that execute the paths $\texttt{path}_i^{n-1}$ at a time $t$. Next, we collect in the set $\mathcal{R}_{\text{\text{new}}}^{n}(t)$ the robots that are new to executing the path $\texttt{path}_i^{n}$ and have not participated in any communication event contained in $\texttt{path}_i^n$ yet. Notice that the robots in $\mathcal{R}^{n-1}(t)$ and $\mathcal{R}_{\text{\text{new}}}^{n}(t)$ have not constructed yet any path $\texttt{path}_i^{n+1,c_i}$.
 Also, we collect in the set $\mathcal{R}_{\text{com}}^{n}(t)$ the robots of all teams $\ccalT_m$, $m\in\ccalM$, that communicate at time $t$ while executing the paths $\texttt{path}_i^{n}$. All other robots that at time $t$ execute the path $\texttt{path}_i^{n}$ but they do not participate in any communication event are collected in the set $\mathcal{R}_{\overline{\text{com}}}^{n}(t)$. Finally, the robots that have already finished the execution of the paths $\texttt{path}_i^{n}$ at time $t$ are collected in the set $\mathcal{R}^{n+1}(t)$. Observe that $\ccalN=\mathcal{R}^{n-1}(t)\cup\mathcal{R}_{\text{\text{new}}}^{n}(t)\cup\mathcal{R}_{\text{com}}^{n}(t)\cup\mathcal{R}_{\overline{\text{com}}}^{n}(t)\cup\mathcal{R}^{n+1}(t),
$
 for all $t\geq 0$, for some $n\geq 0$. Also observe that if $\mathcal{R}^{n+1}(t)\neq\emptyset$, then $\mathcal{R}^{n-1}(t)=\emptyset$, as discussed in Remark \ref{rem:ni}. 

To prove the inequality \eqref{dist1}, we need to define the following cost function:
\begin{align}\label{eq:costopt}
\texttt{cost}(t)&=\sum_{i\in\mathcal{R}_{\text{\text{new}}}^{n}(t)\cup\mathcal{R}^{n-1}(t)} J(\texttt{path}_i^{n} ) + \\&  \sum_{i\in\mathcal{R}_{\text{com}}^{n}(t)} J(\texttt{path}_i^{n+1,c_i(t)} )+ \nonumber\\& \sum_{i\in\mathcal{R}_{\overline{\text{com}}}^{n}(t)} J(\texttt{path}_i^{n+1,c_i(t)} ) + \sum_{i\in\mathcal{R}^{n+1}(t)} J(\texttt{path}_i^{n+1} ),\nonumber
\end{align}
where $\texttt{path}_i^{n+1,c_i(t)} $ denotes the path that has been constructed by Algorithm \ref{alg:plan} by the time instant $t$. Also, note that the robots $i\in\mathcal{R}^{n-1}(t)$ may not have completed the construction of  the paths $\texttt{path}_i^{n} $ yet. Therefore, in the first summation in \eqref{eq:costopt}, the paths $\texttt{path}_i^{n} $ for $i\in\mathcal{R}^{n-1}(t)$, are the ones that these robots will create once they complete their construction. 

Moreover, we define the finite sequence of time instants $\set{t_0^n, t_1^n,\dots,t_{F-1}^n,t_F^n}$, where (i) $t_0^n< \dots<t_{F}^n$, (ii) $t_0^n$ is an arbitrarily selected time instant such that $\mathcal{R}_{\text{\text{new}}}^{n}(t)\cup\mathcal{R}^{n-1}(t)=\ccalN$, (ii)
(iii) $t_F^n$ is the time instant when all robots have completed construction of the paths $\texttt{path}_i^{n+1}$, i.e., $\mathcal{R}^{n+1}(t_F^n)=\ccalN$, and (iv) $t_1^n< \dots<t_{F-1}^n$ are the time instants corresponding to communication events during the execution of any of the paths $\texttt{path}_i^{n}$.\footnote{Note that the time instant $t_0^n$ exists, since it corresponds to a time when the robots either execute paths $\texttt{path}_i^{n_i-1}$ or paths $\texttt{path}_i^{n_i}$ without having participated in any communication events yet; see also Remark \ref{rem:ni}. Also, the sequence $\set{ t_1^n,\dots,t_{F-1}^n,t_F^n}$ for any $n\geq 0$ exists because the network is deadlock-free, as shown in Proposition \ref{prop:deadlock}.} To prove \eqref{dist1}, we need to show that 
\begin{equation}\label{eq:optres1}
\texttt{cost}(t_{k+1}^n)\leq\texttt{cost}(t_k^n),
\end{equation}
for all $k\in\set{0,\dots,F}$.
%
%
%

Since the robots $i\in\mathcal{R}_{\text{\text{new}}}^{n}(t_{k+1}^n)\cup\mathcal{R}^{n-1}(t_{k+1}^n)$ have not constructed  yet any path $\texttt{path}_i^{n_i+1,c_i}$, these robots cannot affect the cost $\texttt{cost}(t_{k}^n)$.
Also, notice that $\texttt{path}_i^{n+1,c_i(t_{k+1}^n)} =\texttt{path}_i^{n+1,c_i(t_{k}^n)}$, for all robots $i\in\mathcal{R}_{\overline{\text{com}}}^{n}(t_{k+1}^n)$, since these robots do not communicate and, therefore, they do not execute Algorithm \ref{alg:plan} at $t_{k+1}^n$. Thus,  the robots  $i\in\mathcal{R}_{\overline{\text{com}}}^{n}(t_{k+1}^n)$ cannot affect the cost $\texttt{cost}(t_{k}^n)$ either. The same holds for the robots $i\in\mathcal{R}^{n+1}(t_{k+1}^n)$. Therefore, for all robots that do not communicate at time $t_{k+1}^n$ it holds that $\sum_{i\in\ccalN\setminus\mathcal{R}_{\text{com}}^{n}(t_{k+1}^n)} J(\texttt{path}_i^{n+1,c_i(t_{k+1}^n)} )= \sum_{i\in\ccalN\setminus\mathcal{R}_{\text{com}}^{n}(t_{k+1}^n)} J(\texttt{path}_i^{n+1,c_i(t_k^n)} )$.
%
In fact, only the robots $i\in\mathcal{R}_{\text{com}}^{n}(t_{k+1}^n)$ that communicate at time $t_{k+1}^n$ design new paths such that $\texttt{path}_i^{n+1,c_i(t_{k+1}^n)} \neq\texttt{path}_i^{n+1,c_i(t_{k}^n)}$. Since $\mathcal{R}_{\text{com}}^{n}(t_{k+1}^n)$ contains all robots that communicate 
at $t_{k+1}^n$ the expression $\sum_{i\in\mathcal{R}_{\text{com}}^{n}(t_{k+1}^n)} J(\texttt{path}_i^{n+1,c_i(t_{k+1}^n)} )$ can be rewritten as follows\footnote{Note that it is possible that two teams  $\ccalT_m$ and $\ccalT_h$ that share at least a robot may be present simultaneously at the same communication point. This can happen, e.g., if the schedule of robot $i\in\ccalT_m\cap\ccalT_h$ has a schedule has the form $\texttt{sched}_i=[m, h,X]^{\omega}$ and $\ccalC_m\cap\ccalC_h\neq\emptyset$. In this case, we assume that communication at the common communication point will happen sequentially across the teams according to the schedules. This ensures that in the second summation in \eqref{eq:sumB}, we never double count the cost of the paths $\texttt{path}_i^{n+1,c_i(t)}$. }
\begin{align}\label{eq:sumB}
&\sum_{i\in\mathcal{R}_{\text{com}}^{n}(t_{k+1}^n)} J(\texttt{path}_i^{n+1,c_i(t_{k+1}^n)} )=\nonumber\\&\sum_{m\in\ccalA(t_{k+1}^n)}\sum_{i\in\ccalT_m}J(\texttt{path}_i^{n+1,c_i(t_{k+1}^n)} ),
\end{align}
\normalsize
where $\ccalA(t)\subseteq\ccalM$ is the set of the teams that communicate at time $t$. By the proof of Proposition \ref{prop:feas}, we get that $\texttt{path}_i^{n+1,c_i(t_k^n)}$ is  a feasible path returned by Algorithm \ref{alg:plan} as a candidate path for $\texttt{path}_i^{n+1,c_i(t_{k+1}^n)}$; it will become  $\texttt{path}_i^{n+1,c_i(t_{k+1}^n)}$ if it also the optimal one.
%
Therefore, we get that
 $\sum_{i\in\ccalT_m}J(\texttt{path}_i^{n+1,c_i(t_{k+1}^n)} )\leq \sum_{i\in\ccalT_m}J(\texttt{path}_i^{n+1,c_i(t_{k}^n)} )$, for all $m\in\ccalA(t_{k+1}^n)$, which implies $\sum_{i\in\mathcal{R}_{\text{com}}^{n}(t_{k+1}^n)} J(\texttt{path}_i^{n+1,c_i(t_{k+1}^n)} )\leq \sum_{i\in\mathcal{R}_{\text{com}}^{n}(t_{k+1}^n)} J(\texttt{path}_i^{n+1,c_i(t_{k}^n)} )$, due to \eqref{eq:sumB}. Therefore,  we get that \eqref{eq:optres1} holds, completing the proof.  \end{proof}

\subsection{Complexity}\label{sec:complex}


In the following proposition, we show that Algorithm \ref{alg:plan} terminates after a finite number of iterations and, therefore, the computational cost is bounded.

\begin{prop}[Convergence]\label{prop:converg2}
There exist iterations $P\leq C$ in Algorithm \ref{alg:plan} so that the sequence
$\texttt{path}_i^P,\texttt{path}_i^{P+1},\dots,\texttt{path}_i^C$ is repeated indefinitely for all $n_i\geq C$ and all  $i\in\ccalN$.
\end{prop}

\begin{proof} 
 %
To show this result, notice that the sets of communication points $\ccalC_m$ are finite, for all $m\in\ccalM$ and, therefore, the number of possible combinations of communication points that can be assigned to the teams is finite. Therefore, there exists an index $n$ where the paths $\texttt{path}_i^{n}$ contain communication points that have appeared in a previous path $n'\leq n$, as well, for all $i\in\ccalN$. Let $C$ be the first index $n$ when it holds that the communication points that appear in the paths $\texttt{path}_i^{C}$ have already appeared in a previous path $\texttt{path}_i^{P-1}$, for some $P\leq C$ and for all $i\in\ccalN$.
Since the selected communication points in the paths $\texttt{path}_i^{P-1}$ and $\texttt{path}_i^{C}$  are the same, we have that Algorithm \ref{alg:plan} generates the same optimal suffix path $\tilde{\tau}_i^{\text{suf},j^*}$ to synthesize  both $\texttt{path}_i^{P-1}$ and $\texttt{path}_i^{C}$. Therefore, we get that $\texttt{path}_i^{P-1}=\texttt{path}_i^{C}$. Consequently, the path $\texttt{path}_i^{C+1}$ will be the same as the path constructed at iteration 
$P$, i.e., $\texttt{path}_i^{C+1}=\texttt{path}_i^{P}$, since the optimal control synthesis problems that are solved to construct the path $\texttt{path}_i^{C+1}$ and $\texttt{path}_i^{P}$ are the same, for all robots $i\in\ccalN$. Similarly, we have that $\texttt{path}_i^{C+2}=\texttt{path}_i^{P+1}$. By inspection of the repetitive pattern, we conclude that for any $n\in\mathbb{N}$ it holds that $\texttt{path}_i^{C+n}=\texttt{path}_i^{C+n-\left(\left\lfloor \left(C+n\right)/\left(C-P+1\right)\right\rfloor-1\right)\left(C-P+1\right)}$,
where $\left\lfloor \cdot\right\rfloor$ stands for the floor function.  We conclude that the sequence $\texttt{path}_i^P,\texttt{path}_i^{P+1},\dots,\texttt{path}_i^C$ is repeated indefinitely for all iterations $n_i\geq C$ of Algorithm \ref{alg:plan} and for all robots $i\in\ccalN$ completing the proof.
which completes the proof.\end{proof}

\begin{rem}[Optimality of Algorithm \ref{alg:plan}]
Notice that Propositions \ref{prop:distance}-\ref{prop:converg2} do not guarantee that Algorithm \ref{alg:plan} will find the optimal prefix-suffix plan that minimizes the cost $J_p(\tau_i)=\alpha\sum_{i\in\ccalN}J(\tau_i^{\text{pre}})+(1-\alpha)\sum_{i\in\ccalN}J(\tau_i^{\text{suf}})$. Instead they only ensure that the total cost $\sum_{i\in\ccalN}J(\texttt{path}_i^{n})$ decreases with every iteration $n$ until $n=P$, when $\sum_{i\in\ccalN}J(\texttt{path}_i^{P})=\sum_{i\in\ccalN}J(\texttt{path}_i^{P+1})=\dots=\sum_{i\in\ccalN}J(\texttt{path}_i^{C})$ while for all iterations $n_i\geq C$ the sequence of paths $\texttt{path}_i^P,\texttt{path}_i^{P+1},\dots,\texttt{path}_i^C$ is repeated indefinitely. Therefore, the best plans $\tau_i^{n_i}$ \eqref{eq:taui} are obtained for any $n_i\geq P$, for all robots $i\in\ccalN$.
%
Sub-optimality is due to the decomposition of Problem \ref{pr:pr1} into intermittent communication control (Section \ref{sec:commun}) and task planning (Section \ref{sec:integration}) that are solved independently. The optimal plan can be found by translating the global LTL formula \eqref{eq:taskandcom} into a NBA, constructing a product automaton across all robots in the network as, e.g., in \cite{ulusoy2013optimality, ulusoy2014optimal}, and using graph search methods to find the optimal plan. However, such centralized methods are computationally expensive and resource demanding as it is also discussed in the Introduction. Moreover, recall that in this work we assume that the teams $\ccalT_m$ are fixed and never change. Note that the total cost of the plans $\tau_i^{n_i}$ can be further minimized if the robots in every team $\ccalT_m$ update not only the communication point $\bbv_j$, $j\in\ccalC_m$, but also the teams they belong to. Optimal design of the teams is part of our future work.
\end{rem}




\section{Simulation Studies}\label{sec:sim}
 In this section, we present a simulation study, implemented using MATLAB R2015b on a computer with Intel Core i7 2.2GHz and 4Gb RAM that illustrates our approach for a network of $N=12$ robots. Robots are categorized into $M=12$ teams as follows: $\ccalT_1=\{1,2,9\}$, $\ccalT_2=\{3,4,5\}$, $\ccalT_3=\{3, 6\}$, $\ccalT_4=\{1,3\}$, $\ccalT_5=\{2,5,6,11\}$, $\ccalT_6=\{4,12\}$, $\ccalT_7=\{5,9\}$, $\ccalT_8=\{4,9,12\}$, $\ccalT_9=\{6,7,10\}$, $\ccalT_{10}=\{7,8,11\}$, $\ccalT_{11}=\{8,10,11,12\}$, and $\ccalT_{12}=\{7,10\}$. Notice that the construction of teams $\ccalT_m$ results in a connected graph $\ccalG_{\ccalT}$ with $\max\{d_{\ccalT_m}\}_{m=1}^M=7$, as discussed in Section \ref{sec:prob}.  Mobility of each robot in the workspace is captured by a wTS with $|\ccalQ_i|=300$ states that represent $W=300$ locations of interest and weights $w_i$ that capture the distance between its states. Among the $W=300$ locations of interest, $R=70$ locations correspond to possible communication points. Also, every team has  $4\leq\left|\ccalC_m\right|\leq 6$ communication points while $\ccalC_m\cap\ccalC_n=\varnothing$, for all $m,n\in\ccalM$. Also, the parameter $\alpha$ in \eqref{eq:cost2} is selected as $\alpha=0.5$.
To model uncertainty in robot mobility, caused by exogenous disturbances that may affect the arrival times of the robots at the communication locations, we assume that the time required for robot $i$ to travel from location $\bbv_e$ to $\bbv_j$, with $(q_i^{\bbv_e},q_i^{\bbv_j})\in\rightarrow_i$, is generated by a uniform distribution on $[1,2]$, at the moment when robot $i$ arrives at location $\bbv_e$. 

The $\text{LTL}_{-\bigcirc}$ tasks for robots $1$, and $3$ are $\phi_1=\square\Diamond(\pi_1^{\bbv_{20}}\vee\pi_1^{\bbv_{10}}\vee\pi_1^{\bbv_{11}} ) \wedge \square\Diamond(\pi_1^{\bbv_{61}})\wedge \square\Diamond(\pi_1^{\bbv_{91}}\vee\pi_1^{\bbv_{100}}\vee\pi_1^{\bbv_{5}} \vee\pi_1^{\bbv_{60}})  \wedge\square (\neg \pi_1^{\bbv_{44}})\wedge \Diamond(\pi_1^{\bbv_{6}}\vee\pi_1^{\bbv_{7}}\vee\pi_1^{\bbv_{133}} )$ and $\phi_3=\square\Diamond(\xi_3^1\vee \xi_3^2)
\wedge[\square\Diamond(\xi_3^3)]\wedge\Diamond[\xi_3^2\rightarrow\square(\neg \xi_3^1) ]\wedge
(\neg\xi_3^3\mathcal{U}\xi_3^1)$, respectively, where $\xi_3^1=\pi_3^{\bbv_{81}}\vee \pi_3^{\bbv_{91}}$, $\xi_3^2=\pi_3^{\bbv_{120}}\vee \pi_3^{\bbv_{91}}\vee\pi_3^{\bbv_{31}}$, and $\xi_3^3=\pi_3^{\bbv_{91}}\vee \pi_3^{\bbv_{110}}\vee \pi_3^{\bbv_{15}}\vee \pi_3^{\bbv_{130}}$.
 All other robots are responsible for similar LTL tasks. For instance, the LTL formula in $\phi_3$ requires robot $3$ to (i) satisfy infinitely often either the Boolean formula $\xi_3^1$ or $\xi_3^2$; (ii) satisfy infinitely often the Boolean formula $\xi_3^3$; (iii) never satisfy $\xi_3^1$ if $\xi_3^2$ is ever satisfied; and (iv) never satisfy $\xi_3^3$ until $\xi_3^1$ is satisfied.
The Boolean formula $\xi_3^1$ is satisfied if robot $3$ visits either $\bbv_{81}$ or $\bbv_{91}$. The Boolean formulas $\xi_3^2$ and $\xi_3^3$ are interpreted similarly. Also, note that robot $1$ is responsible for visiting a user located at $\bbv_{61}$ infinitely often to transmit all collected information.


The schedules of communication events constructed as per Algorithm \ref{alg:schedule} have the following form with length $\ell=4\leq \max\{d_{\ccalT_m}\}_{m=1}^{12}+1=8$.

\footnotesize{
\begin{align}
\texttt{sched}_1&=[1,~4,~X,~X]^{\omega},~\texttt{sched}_7=[9,~12,~10,~X]^{\omega},\nonumber\\
\texttt{sched}_2&=[1,~5,~X,~X]^{\omega},~\texttt{sched}_8=[X,~X,~10,~11]^{\omega},\nonumber\\
\texttt{sched}_3&=[2,~4,~3,~X]^{\omega},~~\texttt{sched}_9=[1,~X,~8,~7]^{\omega},\nonumber\\
\texttt{sched}_4&=[2,~6,~8,~X]^{\omega},~~\texttt{sched}_{10}=[~9,~12,~X,~11]^{\omega},\nonumber\\
\texttt{sched}_5&=[2,~5,~X,~7]^{\omega},~~\texttt{sched}_{11}=[X,~5,~10,~11]^{\omega}.\nonumber\\
\texttt{sched}_6&=[9,~5,~3,~X]^{\omega},~~\texttt{sched}_{12}=[X,~6,~8,~11]^{\omega}.\nonumber
\end{align}}

\normalsize

Then, given the above schedules, feasible initial paths $\texttt{path}_i^0$ are constructed for all robots in $3$ seconds approximately using \cite{kantaros2017Csampling}. Specifically, given communication points for all teams $\ccalT_m$, $m\in\ccalM_i$,  \cite{kantaros2017Csampling} can synthesize a feasible plan $\tilde{\tau}_i^{0}$ that satisfies $\psi_i$ in $0.35$ seconds on average for all $i\in\ccalN$. Similar runtimes are reported if off-the-shelf model checkers, such as  NuSMV \cite{cimatti2002nusmv}, are employed for initialization. 
Moreover, Algorithm \ref{alg:plan} constructs online paths $\texttt{path}_i^{n_i}$ with $P=C=5$. The size of the NBA $B_i$ that corresponds to $\psi_i$ in \eqref{eq:psi} satisfies $7\leq|\ccalQ_{B_i}|\leq 16$, for all $i\in\ccalN$ while the average runtime to solve a single optimal control synthesis problem to generate the optimal suffix path $\tilde{\tau}_i^{\text{suf},j}$ was $45$ seconds. Since $4\leq |\ccalC_m|\leq 6$, for all $m\in\ccalC_m$, the average runtime of Algorithm \ref{alg:plan} per iteration $c_i$ is between $4\times 45=180$ seconds and $6\times 45=270$ seconds. Note that this runtime depends only on the size of the sets $\ccalC_m$ and not on the size of the teams $\ccalT_m$.  Note also that this runtime is higher than the initialization runtime, since during initialization only feasible plans are required while for the online construction of $\texttt{path}_{i}^{n_i}$ optimal suffix paths are created.  More computational efficient methods are discussed in Remark \ref{rem:comp} that can decrease the corresponding runtime.

To illustrate that the designed motion plans ensure intermittent communication among the robots infinitely often, we implement a consensus algorithm over the dynamic network $\ccalG_c$. Specifically, we assume that initially all robots generate a random number $v_{i}(t_0)$ and when all robots $i\in\ccalT_m$ meet at a communication point $j\in\ccalC_m$ they perform the following consensus update $v_{i}(t)=\frac{1}{\left|\ccalT_m\right|}\sum\nolimits_{e\in\ccalT_m}v_{e}(t)$. 
Figure \ref{fig:cons} shows that eventually all robots reach a consensus on the numbers $v_i(t)$, which means that communication among robots takes place infinitely often, as proven in Theorem \ref{thm:com}. Moreover, Figure \ref{fig:tn} shows the time instants when robots $1,~2,$ and $3$ started executing the paths $\texttt{path}_i^{n}$, for all $n\in\set{1,\dots,14}$. Observe in Figure \ref{fig:tn} that there exist time instants $t_n$ when all three robots are executing their respective paths $\texttt{path}_i^{n}$ for a common $n$, for all $n\in\set{1,\dots,14}$, as discussed in Remark~\ref{rem:ni}. The communication events over time for teams $\ccalT_1$ and $\ccalT_5$ are depicted in Figure \ref{fig:meetings}. Observe in
Figure \ref{fig:meetings} that the communication time instances do not depend linearly on time, which means that communication within these teams is aperiodic. Figure \ref{fig:distance} shows that the total traveled distance $\sum_{i=1}^{N}J(\texttt{path}_i^{n})$ with respect to $n\in\mathbb{N}$ which decreases as expected due to Proposition~\ref{prop:distance}.  The corresponding simulation video can be found in \cite{LTLvideo}.


\begin{figure}[t]
  \centering
  \subfigure[]{
  \label{fig:cons}
    \includegraphics[width=0.46\linewidth]{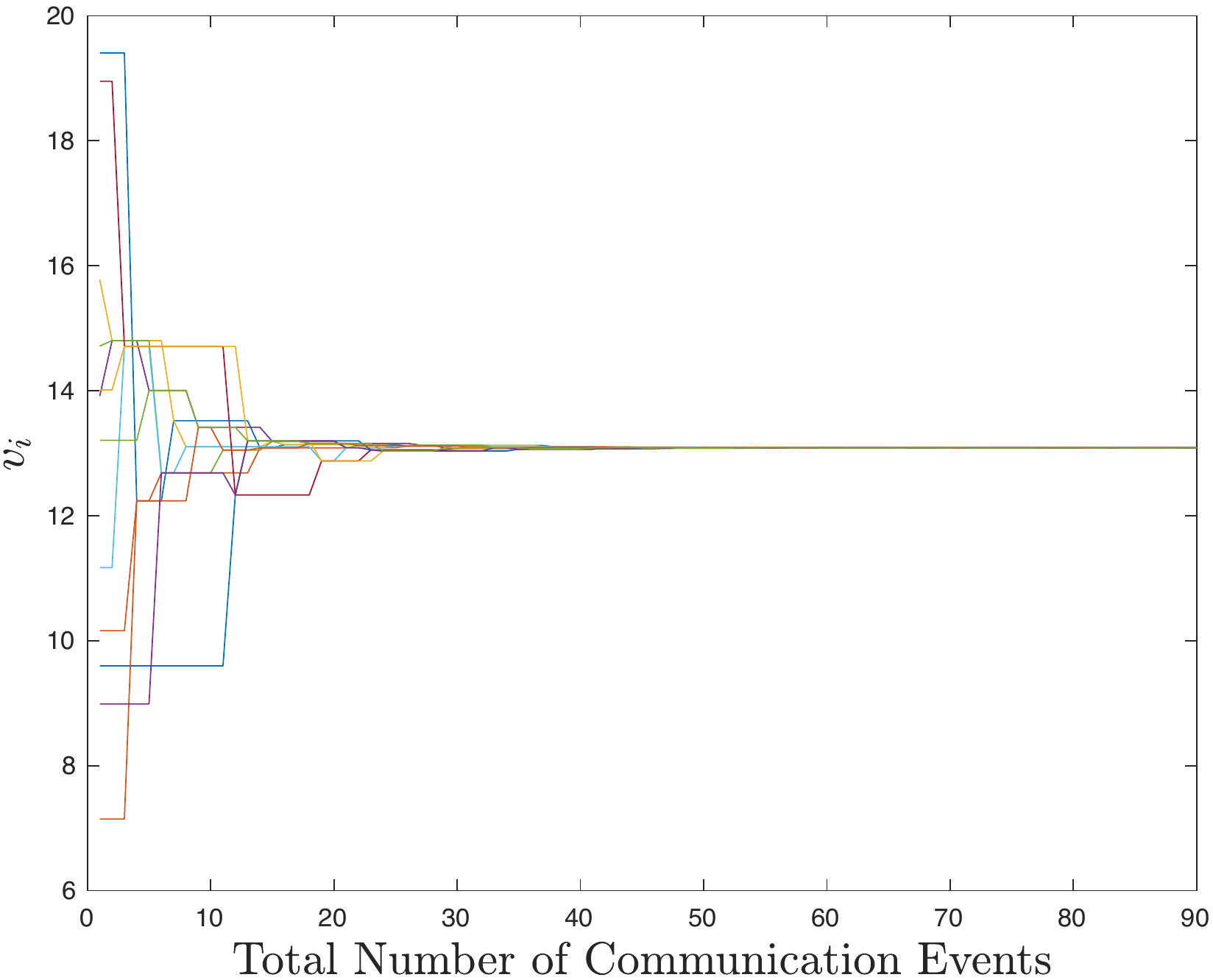}}
      \subfigure[]{
  \label{fig:tn}
    \includegraphics[width=0.47\linewidth]{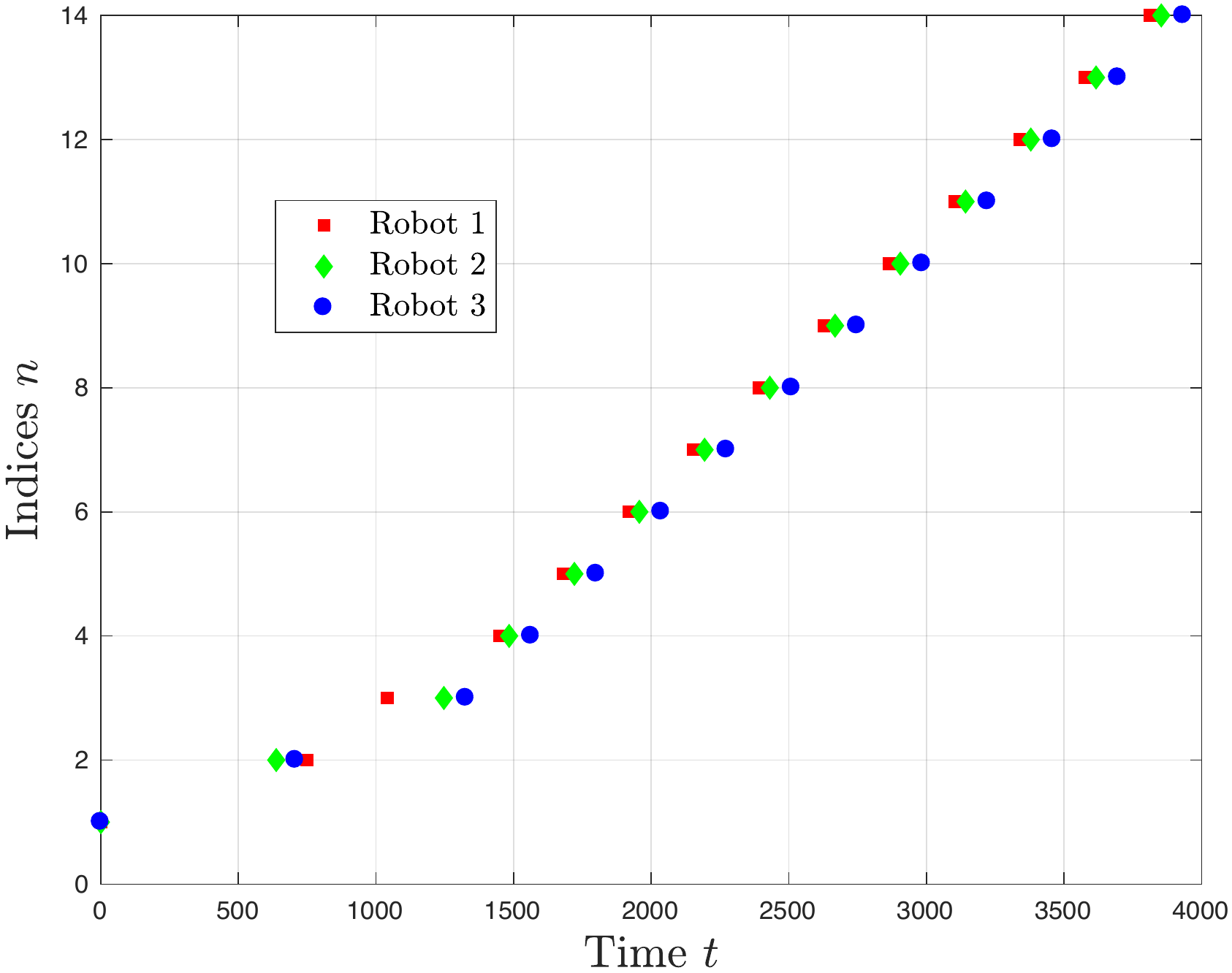}}
  \caption{Figure \ref{fig:cons} depicts the consensus of numbers $v_{i}(t)$. Figure \ref{fig:tn} illustrates the time instants when the robots $1,2,$ and $3$ started executing the paths $\texttt{path}_i^n$. For instance, the time between the second and the third red square denotes the time required by robot $1$ to travel along the path $\texttt{path}_1^2$.}
  \label{fig:two}
\end{figure}

\begin{figure}[t]
    \centering
    \subfigure[Team 1]{
    \label{team1}
    \includegraphics[width=0.44\linewidth]{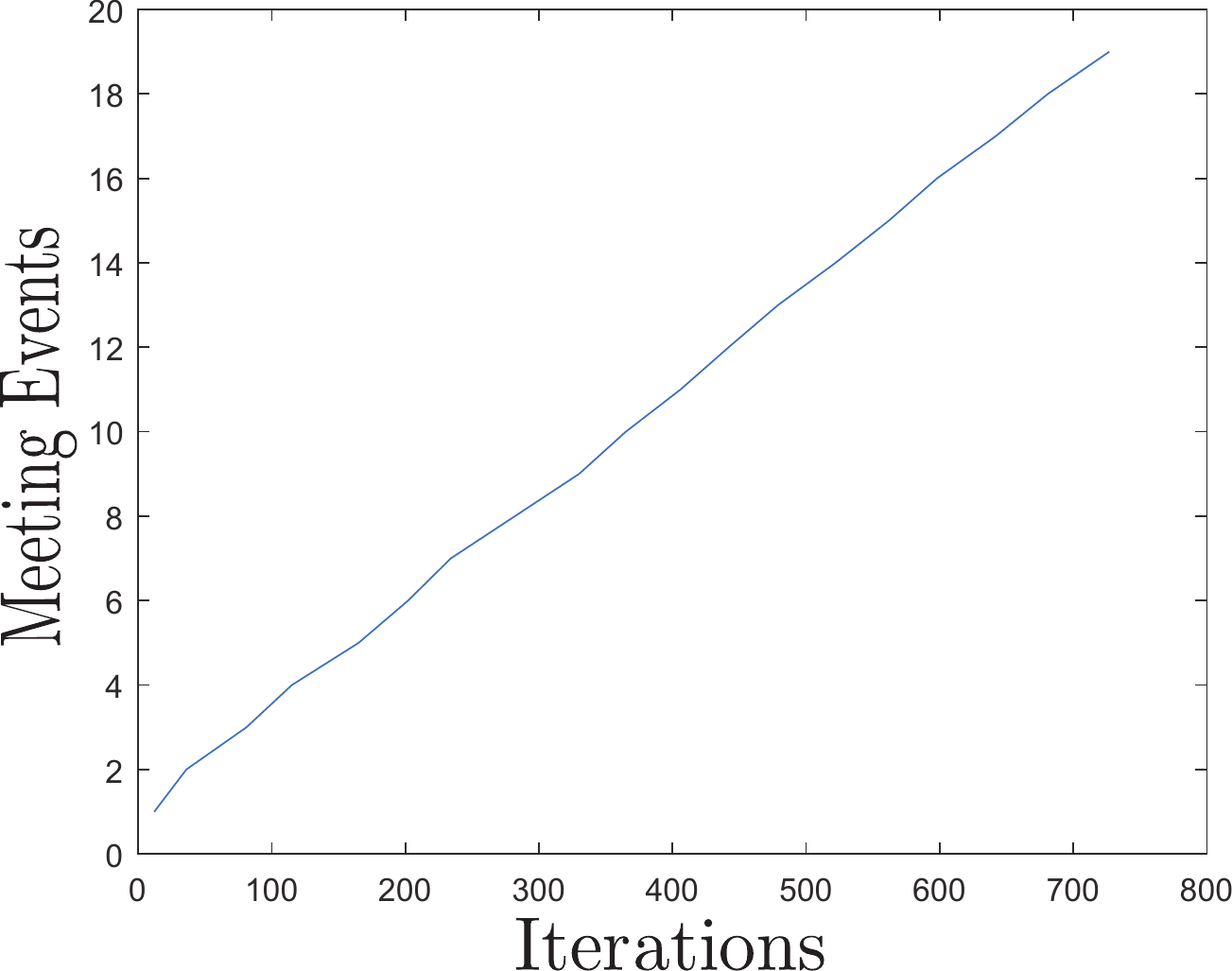}} 
    \subfigure[Team 4]{
    \label{team5}
       \includegraphics[width=0.44\linewidth]{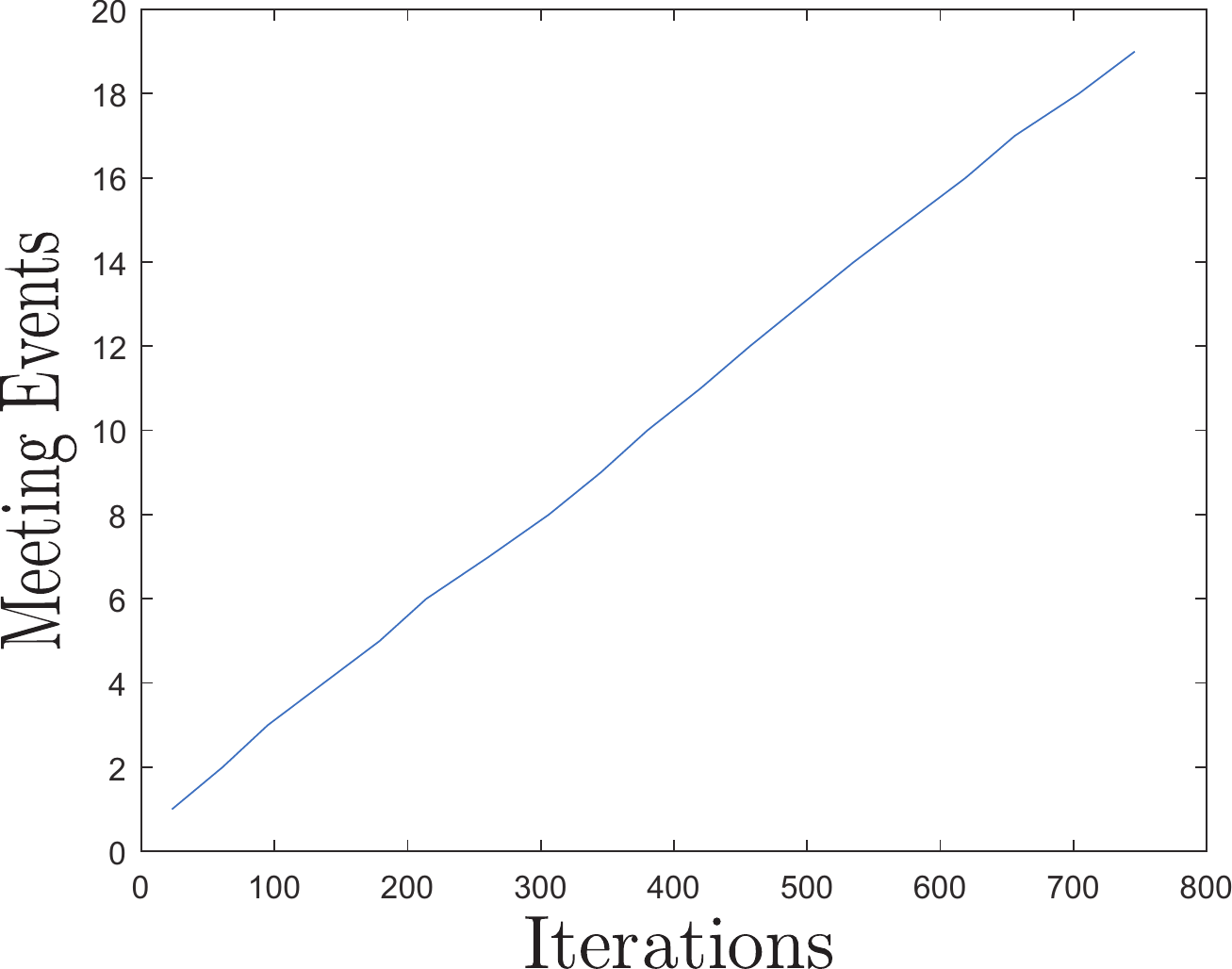}}                                
        \caption{Graphical depiction of communication events for team $\ccalT_1$ (Figure \ref{team1}) and $\ccalT_4$ (Figure \ref{team5}) with respect to time. }
     \label{fig:meetings}
\end{figure}

\begin{figure}[t]
  \centering
    \includegraphics[width=0.5\linewidth]{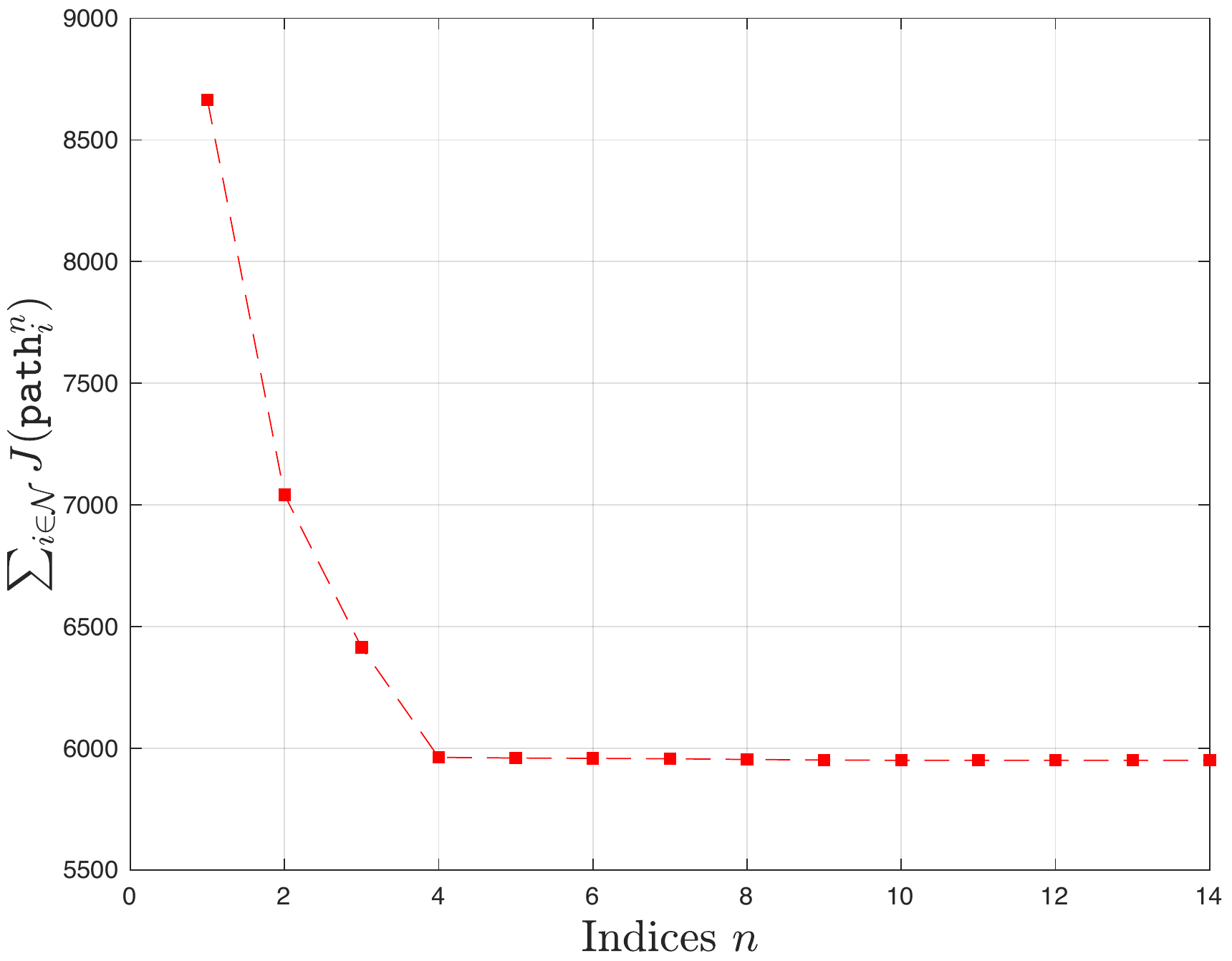}\\
  \caption{Evolution of the total cost $\sum_{i=1}^{12}J(\texttt{path}_i^n)$ with respect to iterations $n$. Note that there is also a slight decrease in $\sum_{i=1}^{12}J(\texttt{path}_i^n)$ from $n=4$ to $n=5$. After $n=5$, a repetitive pattern in $\texttt{path}_i^n$ is detected giving rise to motion plans $\tau_i$ in a prefix-suffix form. }
  \label{fig:distance}
\end{figure}


Note also that due to excessive memory requirements it would be impossible to generate optimal motion plans ${\tau}_{i}$ by using either the optimal control synthesis methods presented in \cite{kantaros15asilomar,ulusoy2013optimality,ulusoy2014optimal} that rely on the construction of a synchronous product automaton or off-the-shelf model checkers \cite{cimatti2002nusmv,holzmann2004spin} that can construct feasible but not optimal paths. Specifically, \cite{kantaros15asilomar,ulusoy2013optimality,ulusoy2014optimal} rely on the construction of a product transition system (PTS), whose state space has dimension $|\mathcal{Q}_{\text{PTS}}|=\times_{\forall i}|\ccalQ_{i}|=W^{|N|}=300^{12}=5.3144\times 10^{29}$. This PTS is combined with the B$\ddot{\text{u}}$chi Automaton $B$ that corresponds to the LTL statement $\phi=(\wedge_{\forall i\in\ccalN}\phi_i)\wedge\phi_{\text{com}}$ to construct a Product B$\ddot{\text{u}}$chi Automaton whose state space has dimension $|\mathcal{Q_{\text{PBA}}}|=|\mathcal{Q}_{\text{PTS}}|\times|\ccalQ_{B}|=5.3144\times 10^{29}\times|\ccalQ_B|$ which is too large to manipulate in practice let alone searching for an optimal accepting infinite run. Finally, we validated the efficacy of the proposed distributed algorithm by experimental results that are omitted due to space limitations. The video showing the conducted experiment along with its description can be found in~\cite{int_com_video}.

\section{Conclusion}\label{sec:conclusion}
In this paper, we developed the first distributed and online intermittent communication framework for networks of mobile robots with limited communication capabilities that are responsible for accomplishing temporal logic tasks. Our proposed distributed online control framework jointly determines local plans that allow all robots to fulfill their assigned $\text{LTL}_{-\bigcirc}$ tasks, schedules of communication events that guarantee information exchange infinitely often, and optimal communication locations that minimize a desired distance metric. We showed that the proposed method can solve optimally very large-scale problems that are impossible to solve using current off-the-shelf model-checkers.

\bibliographystyle{IEEEtran}
\bibliography{YK_bib}

\end{document}